\documentclass{article}

\usepackage{microtype}
\usepackage{graphicx}
\usepackage{subfigure}
\usepackage{booktabs} 

\usepackage{hyperref}

\usepackage[nothing]{icml2025}

\usepackage{amsmath}
\usepackage{amssymb}
\usepackage{mathtools}
\usepackage{amsthm}

\usepackage[capitalize,noabbrev]{cleveref}

\usepackage{multirow}
\usepackage{makecell}
\usepackage{float}

\theoremstyle{plain}
\newtheorem{theorem}{Theorem}[section]

\newtheorem{lemma}[theorem]{Lemma}

\theoremstyle{definition}

\theoremstyle{remark}

\usepackage[textsize=tiny]{todonotes}

\usepackage{tikz}
\usepackage{verbatim}
\usepackage{hyperref}
\usepackage{harpoon}
\usepackage{halloweenmath}
\usepackage{mathabx}
\usepackage{enumitem}

\newcommand{\dee}{\mathrm{d}}

\newcommand{\etal}{{et~al.}}

\newcommand{\ra}{\rightarrow}
\newcommand{\mcal}[1]{\mathcal{#1}}

\newcommand{\x}{\mathbf{x}}

\newcommand{\tr}{\mathrm{tr}}

\newcommand{\Real}{\mathbb{R}}

\newcommand{\data}{\mathrm{data}}

\definecolor{cellhead}{HTML}{F4F4F4}
\newcommand{\mc}{\cellcolor{cellhead}\rule{0pt}{2.5ex}}
\definecolor{grayline}{HTML}{DDDDDD}
\icmltitlerunning{Distilling Two-Timed Flow Models by Separately Matching Initial and Terminal Velocities}

\begin{document}

\twocolumn[

\icmltitle{Distilling Two-Timed Flow Models by\\Separately Matching Initial and Terminal Velocities}



\icmlsetsymbol{equal}{*}

\begin{icmlauthorlist}
\icmlauthor{Pramook Khungurn}{xxx}
\icmlauthor{Pratch Piyawongwisal}{yyy}
\icmlauthor{Sira Sriswasdi}{zzz}
\icmlauthor{Supasorn Suwajanakorn}{yyy}
\end{icmlauthorlist}

\icmlaffiliation{xxx}{pixiv Inc., Tokyo, Japan}
\icmlaffiliation{yyy}{School of Information Science and Technology, VISTEC, Rayong, Thailand}
\icmlaffiliation{zzz}{Center for Artificial Intelligence in Medicine, Research Affairs, Faculty of Medicine, Chulalongkorn University, Bangkok, Thailand}

\icmlcorrespondingauthor{Pramook Khungurn}{pong@pixiv.co.jp}

\icmlkeywords{Machine Learning, diffusion models, flow matching models, distillation}

\vskip 0.3in
] 



\printAffiliationsAndNotice{}  

\begin{abstract}
A flow matching model learns a time-dependent vector field $v_t(x)$ that generates a probability path $\{ p_t \}_{0 \leq t \leq 1}$ that interpolates between a well-known noise distribution ($p_0$) and the data distribution ($p_1$). It can be distilled into a \emph{two-timed flow model} (TTFM) $\phi_{s,x}(t)$ that can transform a sample belonging to the distribution at an initial time $s$ to another belonging to the distribution at a terminal time $t$ in one function evaluation. We present a new loss function for TTFM distillation called the \emph{initial/terminal velocity matching} (ITVM) loss that extends the Lagrangian Flow Map Distillation (LFMD) loss proposed by Boffi \etal\ \yrcite{Boffi:FMM:2024} by adding redundant terms to match the initial velocities at time $s$, removing the derivative from the terminal velocity term at time $t$, and using a version of the model under training, stabilized by exponential moving averaging (EMA), to compute the target terminal average velocity. Preliminary experiments show that our loss leads to better few-step generation performance on multiple types of datasets and model architectures over baselines. 
\end{abstract}

\section{Introduction}

Diffusion models \cite{SohlDickstein:2015,Ho:2020} model a continuum of probability distributions $\{ p_t \}_{0 \leq t \leq 1}$ such that $p_0$ is a well-known noise distribution, and $p_1$ is the data distribution $p_{\data}$ or a slightly noisy version of it.\footnote{Lipman \etal\ call this a ``probability path'' \yrcite{Lipman:2023}.} A data item can be sampled by first sampling a noise vector from $p_0$ and then gradually transforming it into samples from $p_{t_1}$, $p_{t_2}$, $\dotsc$, $p_{t_K}$ where $0 < t_1 < t_2 < \dotsb < t_K = 1$. Compared to the previous reigning champion, GANs \cite{Goodfellow:2014}, diffusion models yield better sample quality \cite{Dhariwal:2021} and are more stable to train. Nevertheless, because the sampling process involves applying the model multiple times, diffusion models are significantly slower than other approaches, such as GANs, VAEs \cite{Kingma:2014}, and normalizing flows \cite{Rezende:2016} that can generate a sample with a single model evaluation. As a result, speeding up diffusion models has received much interest.

One model acceleration approach is to distill \cite{Hinton:2015} the knowledge of a pre-trained diffusion model into a new network that can generate samples in one or a few evaluations. We are interested in distilling a flow matching model \cite{Lipman:2023}, a variant of diffusion models that predicts a time-dependent velocity field that continuously transforms $p_0$ to $p_1$. Specifically, our objective is to construct a \emph{two-timed flow model} (TTFM)\footnote{Different authors have different names for this type of function. Lipman \etal\ call it a ``flow'' \yrcite{Lipman:2023}, Boffi \etal\ a ``flow map'' \cite{Boffi:FMM:2024}, and Kim \etal\ a ``consistency trajectory model'' \yrcite{Kim:CTM:2023}. We think that our naming better captures its characteristics and distinguishes it from other similar functions such as a normalizing flow, which has no time input, and a consistency model \cite{Song:CM:2023}, which has one time input.}  $\phi_{s,t}(x)$, which is a deterministic function such that $\phi_{s,t}(x) \sim p_t$ given that $x \sim p_s$ for any $0 \leq s < t \leq 1$. TTFM allows a data item to be generated in a single evaluation, by computing $\phi_{0,1}(x_0)$ where $x_0 \sim p_0$, or in multiple steps, which may increase the sample quality. Furthermore, TTFM also supports sample manipulation, such as with SDEdit \cite{Meng:SDEdit:2022}. While several approaches train the student models with no time input \cite{Luhman:2021, Zheng:DFNO:2023, yin:dmd:2024, yin:dmd2:2024, Xie:EM:2024, Kang:Diffusion2GAN:2024, Nguyen:SwiftBrush:2024} or one time input \cite{Salimans:2022, Sauer:ADD:2023, Song:CM:2023, Berthelot:TRACT:2023, Tee:2024:PID}, only a few have considered two time inputs \cite{Kim:CTM:2023, Boffi:FMM:2024}.

In this paper, we explore TTFM distillation and propose a new loss function for it, called the \emph{initial/terminal velocity matching} (ITVM) loss. It is based on the recent Lagrangian flow matching distillation (LFMD) loss \cite{Boffi:FMM:2024} that forces the trained network to obey a partial differential equation (PDE) that governs a flow matching model's trajectory. 
We make a key observation that, rather than enforcing the PDE directly, we can enforce special cases of the PDE along with a property called ``consistency'' \cite{Song:CM:2023} to achieve the same effect. This formulation has an advantage that, during training, the input to the teacher model is always in its training distribution, resulting in better supervising signals. More concretely, we modify the LFMD loss by first adding mathematically redundant terms that enforce the special cases. We then convert the original formula by replacing the partial derivative with a finite difference, which simplifies the implementation and speeds up computation. We then change the matching target of the finite difference from the teacher model to the exponentially moving averaged version of the model under training. This last step encourages the student to be consistent with itself rather than reproduce the teacher's out-of-distribution outputs. We extensively evaluated ITVM under multiple datasets and model architectures and found that it yields better few-step generation performance on many datasets than those of the LFMD loss and other baselines.


\section{Previous Works}

\textbf{Diffusion models} \cite{SohlDickstein:2015,Ho:2020} have demonstrated remarkable effectiveness in generative modeling across diverse domains such as image, audio, and video. Their approach learns the score of a noise-adding process governed by a stochastic differential equation (SDE), which has an equivalent ordinary differential equation (ODE), and reverses the process by simulating one of the differential equations backward in time \cite{Song:Score:2021}. Recent advancement like rectified flow \cite{Liu:2022}, flow matching \cite{Lipman:2023}, and stochastic interpolants \cite{Albergo:interp:2023} extend this framework by introducing different mathematical formulations and more flexibility in the modeling of the noise-adding process. However, the iterative sampling process of diffusion model is computationally intensive. To address this limitation, two main strategies have emerged: (1) reducing the simulation cost of ODEs or SDEs \cite{Lu:dpmsolver:2022,Shaul:bespoke:2023}, and (2) distillation, which trains a new model (the student) that compresses the sampling process of a diffusion model (the teacher) into fewer steps.

\textbf{Distillation.} A straightforward distillation approach is to match the student's outputs to those of the teacher. Luhman and Luhman proposed a simple implementation \yrcite{Luhman:2021}, which was later significantly improved with better similarity metrics and an adversarial loss \cite{Kang:Diffusion2GAN:2024}. Zheng \etal\ devises a student model that operates in the frequency domain and generates the entire sampling trajectory, which is matched to the teacher's, instead of just the sample \yrcite{Zheng:DFNO:2023}. These works require input-output pairs of the teacher as training data, which can be costly to prepare. 

Progressive distillation (PD) avoids this problem by an iterative distillation process where the student is trained to halve the number of steps of the teacher \cite{Salimans:2022}. Enforcing an additional property called ``consistency'' (to be discussed later) enables training in one round. Examples of this approach include consistency model (CM) \cite{Song:CM:2023,Song:ICM:2023} and consistency trajectory model (CTM) \cite{Kim:CTM:2023}. TRACT is a hybrid between CM and PD that significantly reduces the number of training rounds \cite{Berthelot:TRACT:2023}. 

Another popular approach is to minimize the KL divergence between the sample distributions induced by the teacher and the student \cite{yin:dmd:2024, yin:dmd2:2024, Luo:DiffInstruct:2023, Xie:EM:2024}. Doing so requires access to the student's score, which the model does not directly provide. As such, it is common for an auxiliary network to be trained along with the student to approximate its score. Salimans \etal\ sidestep this by approximating the auxiliary network's parameters from those of the teacher model \yrcite{Salimans:MM:2024}.

Our loss is based on a recent approach that make the student satisfy the PDE that defines the trajectory of an ODE solver when applied to the teacher model. Methods employing this approach include Lagrangian flow map distillation \cite{Boffi:FMM:2024} and physics-informed distillation \cite{Tee:2024:PID}.



\section{Preliminary} \label{sec:preliminary}

We denote a data item by $x \in \Real^d$.  The data distribution is denoted by $p_{\data}$. Following Lipman \etal~\yrcite{Lipman:2023}, given $x_{\data} \sim p_{\data}$, we can define the \emph{conditional probability path} $p_t(x|x_{\data})$ to be $\mcal{N}(tx_{\data}, \sigma_{t}^2 I)$ where $t \in [0,1]$, $\sigma_t = 1 - (1 - \sigma_{\min})t$, and $\sigma_{\min}$ is a small positive constant. It yields the \emph{marginal probability path} $$p_t(x) = \int p(x|x_{\data})p_{\data}(x_{\data})\, \dee x_{\data}$$ such that $p_0 = \mcal{N}(0,I)$, $p_1 = p_{\data} * \mcal{N}(0,\sigma_{\min}^2 I)$, and $p_t$ smoothly interpolates between $p_0$ (a well-known noise distribution) and $p_1$ (a slightly noisy version of $p_{\data}$).

The \emph{optimal transport conditional vector field} is given by $$v_t(x|x_{\data})  = (x_{\data} - (1-\sigma_{\min})x) / (1 - (1-\sigma_{\min})t),$$ and integrating it yields the \emph{marginal vector field} (MVF) $$v_t(x) = \int v_t(x|x_{\data}) p(x|x_{\data})p(x_{\data})/p_t(x)\, \dee x_{\data}.$$ It can be shown that the MVF is equivalent to the \emph{expected velocity field} \cite{Liu:2022}
$$v_t(x) = E[X_1 - X_0 | (1-t)X_0 + tX_1 = x],$$ where $X_0 \sim p_0$, and $X_1 \sim p_1$. We define a \emph{(two-timed) flow} $\phi_{s,t}(x)$ to be a solution to the \emph{Lagrangian PDE} $$\frac{\partial}{\partial t} \phi_{s,t}(x) = v_t(\phi_{s,t}(x)),$$ for all $0 \leq s \leq t \leq 1$ with the initial condition $\phi_{s,s}(x) = x$ for all $0 \leq s \leq 1$. The flow satisfies $p_t = [\phi_{s,t}]_*\, p_s$, where $[\phi_{s,t}]_*$ denotes a \emph{push-forward} by $\phi_{s,t}(\cdot)$ taken as a function from $\Real^d$ to $\Real^d$. We say that the MVF $v_t$ and the flow $\phi_{s,t}$ \emph{generates} the probability path $p_t$.

A \emph{flow matching model} is a neural network $v_t^\eta(x)$ with parameters $\eta$ that models the marginal vector field $v_t(x)$. It can be trained with the \emph{conditional flow matching} (CFM) loss \cite{Lipman:2023}
\begin{align}
\mcal{L}_{\operatorname{CFM}} = E_{t,x_0,x_{\data}}[\| v^\eta_t(x_t) - v_t(x_t|x_{data}) \|^2], \label{cfm}
\end{align}
where $t \sim \mcal{U}[0,1]$, $x_0 \sim p_0$, $x_{\data} \sim p_{\data}$, and $x_t = (1 - (1-\sigma_{\min})t)x_0 + tx_{\data}$ is a noised data distributed according to $p_t$. Given such a model, an \emph{implicit flow} $\phi_{s,t}^\eta(x)$ can be defined by integrating it: setting $\phi_{s,s}^\eta(x) = x$ and
\begin{align*}
    \phi_{s,t}^\eta(x) &= x + \int_s^t v_u(\phi^\eta_{s,u}(x))\, \dee u \quad \mbox{when } s < t.
\end{align*}
In practice, an ODE solver is used to compute the integral.

A \emph{two-time flow model} (TTFM) is a neural network $\phi_{s,t}^\theta(x)$ with parameters $\theta$ that models the (idealized) flow $\phi_{s,t}(x)$. We consider the problem of training a TTFM by distilling a flow matching model $v^{\eta}_t(x)$ so that it mimics the implicit flow $\phi^{\eta}_{s,t}(x)$. A TTFM should satisfy the following three properties. Firstly, it should satisfy the \emph{initial condition}: $\phi_{s,s}^\theta(x) = x$ for all $s \in [0,1]$. Secondly, it should satisfy the \emph{Lagrangian PDE}: $\partial \phi_{s,t}^\theta(x) / \partial t = v_t(\phi_{s,t}^\theta(x)).$ Lastly, it should satisfy \emph{consistency}: $\phi_{s,t}^\theta(x) = \phi_{u,t}^\theta(\phi_{s,u}^\theta(x))$ for any $0 \leq s \leq u \leq t \leq 1$.

{\bf Losses for distillation.} To make a TTFM satisfy the initial condition, we parameterize it as $\phi_{s,t}^\theta(x) =  x + (t-s) v_{s,t}^{\theta}(x)$. The neural network $v^{\theta}_{s,t}(x)$ is called the \emph{average velocity model} (AVM) because it models the average velocity of the particle starting at $x$ at time $s$ and traveling according to the MVF until time $t$. This parameterization is simpler than the ansatz proposed by Boffi \etal~\yrcite{Boffi:FMM:2024} and does not have a risk of division by zero like that of CTM \cite{Kim:CTM:2023}.

We now review some previous loss functions that can be used to distill a flow matching model into a TTFM. Song \etal\ propose the \emph{consistency matching} (CM) loss \yrcite{Song:CM:2023}:
\begin{align*}
\mcal{L}_{\operatorname{CM}} = E_{s,t,x_s}[ \mathfrak{d}(\phi_{s,t}^\theta(x_s), \phi^{\langle \theta \rangle}_{s+\tau,t}(\phi^{\eta}_{s,s+\tau}(x_s))) ],
\end{align*}
where $s \sim \mcal{U}[0,1-\tau]$, $t \sim \mcal{U}[s+\tau,1]$, and $x_s \sim p_s$. The function $\mathfrak{d}(\cdot,\cdot)$ is a distance function on $\Real^d$. $\tau$ is a hyperparameter that is a positive constant chosen so that it is small enough for $\phi^{\eta}_{s,s+\tau}(x)$ to be accurately calculated in one step of an ODE solver. $\langle \theta \rangle$ is the \emph{exponential moving average} (EMA) of the TTFM's parameters. 
We note that the CM loss is designed to be used in situations where $t$ is always fixed at $1$, but it can take all values in $[s+\tau,1]$.

Kim \etal\ proposes the \emph{soft consistency matching} (SCM) loss \yrcite{Kim:CTM:2023}, which generalizes the CM loss by changing the intermediate time $s + \tau$ to the value $u$ sampled from the interval $[s,t]$ and using multiple ODE solver steps to evaluate the term that involves $\phi^{\eta}$. More concretely,
\begin{align*}
\mcal{L}_{\operatorname{SCM}} = E_{s,t,u,x_s}[ \mathfrak{d}( \phi_{s,t}^\theta(x_s), \phi^{\langle \theta \rangle}_{u,t}(\phi^{\eta}_{s,u}(x_s))) ],
\end{align*}
where $u \sim \mcal{U}[s,t]$, and $s$, $t$ and $x$ are sampled in the same way as in the CM loss. We note that Kim \etal\ uses two additional losses to train their CTMs, but we specifically consider SCM, as it is the only one involving distillation.

Boffi \etal\ proposes two loss functions for training TTFMs \yrcite{Boffi:FMM:2024}. The first is the \emph{Lagrangian flow map distillation} (LFMD) loss:
\begin{align*}
\mcal{L}_{\operatorname{LFMD}} = E_{s,t,x_s}\bigg[ \bigg\| \frac{\partial}{\partial t} \phi^\theta_{s,t}(x_s) - v^\eta_t(\phi_{s,t}^{[\theta]}(x_s)) \bigg\|^2 \bigg],
\end{align*}
where $[\theta]$ is the stop-gradient version of the model parameters.\footnote{The use of the stop-gradient operator is not documented in Boffi \etal's paper. However, according to Tee \etal, it led to much better results \yrcite{Tee:2024:PID}.}
The second is the \emph{Eulerian flow map distillation} (EFMD) loss:
\begin{align*}
\mcal{L}_{\operatorname{EFMD}} = E_{s,t,x_s}\bigg[ \bigg\| \frac{\partial}{\partial s} \phi^\theta_{s,t}(x_s) - \frac{\partial}{\partial x_s} \phi_{s,t}^\theta(x_s)\ v^\eta_s(x_s) \bigg\|^2 \bigg].
\end{align*}
While the use of either loss function should result in the same network due to the theoretical properties of Physics Informed Neural Networks \cite{Raissi:2019:PINN}, Boffi \etal\ observes that LFMD performs much better than EFMD in practice. 

Tee \etal\ proposes the \emph{Physics Informed Distillation} (PID) loss that is similar to LFMD but formulated for a standard diffusion model rather than a flow matching model \yrcite{Tee:2024:PID}. When translated to the flow matching formulation, it becomes
\begin{align*}
\mcal{L}_{\operatorname{PID}} 
= E_{s,t,x_s}\bigg[ \mathfrak{d} \bigg( \frac{\phi^\theta_{s,t}(x_s) - \phi^\theta_{s,t-\tau}(x_s)}{\tau}, v^\eta_t(\phi_{s,t}^{[\theta]}(x_s)) \bigg) \bigg],
\end{align*}
where $\tau$ is a small positive constant, and $\mathfrak{d}$ is a distance function. Note that PID replaces the partial derivative term in LFMD with a numerical differentiation. This makes PID simpler to implement because it does not require advanced automatic differentiation features, which LFMD needs.

\section{Method} \label{sec:method}


{\bf Motivation and intuition.} As noted earlier, a valid TTFM must satisfy the initial condition, the Lagrangian PDE, and consistency. The most important property is the Lagrangian PDE because it defines the sampling trajectory. LFMD and PID directly enforce it. In contrast, CM and SCM enforce consistency instead, but the trained network ultimately also satisfies the PDE due to the following fact. 

\begin{lemma} \label{thm:consistency-extension-1}
Suppose there exists $\tau^* > 0$ such that
\begin{enumerate}
    \item[(a)] $\phi^\theta_{s,s+\tau}(x) = \phi^{\eta}_{s,s+\tau}(x)$, and
    \item[(b)] $\phi^\theta_{s,t}(x) = \phi^\theta_{t-\tau,t}(\phi^\theta_{s,t-\tau}(x))$    
\end{enumerate}
for all $x \in \Real^d$, $0 \leq s < t \leq 1$, and $\tau \leq \min\{ \tau^*, t-s \}$.
Then, $\phi^{\theta}_{s,t}(x) = \phi^{\eta}_{s,t}(x)$ and so satisfies Lagrangian PDE.    
\end{lemma}

In other words, if a TTFM agrees with the implicit flow on short intervals, then consistency can be used to extend the interval's length arbitrarily. The use of the teacher model in CM and SCM folds the enforcement of (a) into an expression that enforces (b). Note that in (b), the consistency is required to hold only for $u$ that is close enough to $t$, not for any $u \in [s,t]$.

Taking the limit as $\tau \ra 0$, (a) becomes $v^{\theta}_{s,s}(x) = v^{\eta}_s(x)$. This is a special case of the Lagrangian PDE where $t = s$. Moreover, if it holds together with consistency, then $\phi^{\theta}$ satisfies the Lagrangian PDE.

\begin{lemma} \label{thm:consistency-extension-2}
    Suppose there exists $\tau^* > 0$ such that
    \begin{enumerate}
        \item[(A)] $v^\theta_{s,s}(x) = v^\eta_{s}(x)$, and
        \item[(B)] $\phi^\theta_{s,t}(x) = \phi^\theta_{t-\tau,t}(\phi^\theta_{s,t-\tau}(x))$.
    \end{enumerate}
    for all $x \in \Real^d$, $0 \leq s \leq t \leq 1$, and $\tau \leq \min\{\tau^*, t-s\}$.
    Then, $\phi^\theta$ satisfies the Lagrangian PDE.    
\end{lemma}

Lemma~\ref{thm:consistency-extension-1} and \ref{thm:consistency-extension-2} together indicate that a TTFM may be trained by enforcing consistency and special cases of the Lagrangian PDE, instead of enforcing the PDE itself.

{\bf Our proposal.} Our loss function, the \emph{initial/terminal velocity matching} (ITVM) loss, is designed to enforce properties in Lemma \ref{thm:consistency-extension-1} and \ref{thm:consistency-extension-2}. It has three terms: 
\begin{align*}
\mcal{L}_{\operatorname{ITVM}} = \mcal{L}_{\operatorname{IIVM}}  + \mcal{L}_{\operatorname{IAVM}} + \mcal{L}_{\operatorname{TVM}}.
\end{align*}
The first term, the \emph{initial instantaneous velocity matching} (IIVM) loss, 
aims to enforce Property (A) of Lemma~\ref{thm:consistency-extension-2}:
\begin{align*}
\mcal{L}_{\operatorname{IIVM}} = E_{s,x_s} [ \| v^\theta_{s,s}(x_s) - v^\eta_s(x_s) \|^2 ],
\end{align*}
where $s \sim \mcal{U}[0,1]$ and $x_s \sim p_s$.

The second term, the \emph{initial average velocity matching} (IAVM) loss, aims to enforce Property (b) of Lemma~\ref{thm:consistency-extension-1}.
\begin{align*}
\mcal{L}_{\operatorname{IAVM}} 
= E_{s,x_s} \bigg[ \bigg\| v^\theta_{s,s+\tau}(x_s) - \frac{\phi^{\eta}_{s,s+\tau}(x_s)-x_s}{\tau} \bigg\|^2 \bigg]. 
\end{align*}
Here, $\tau$ is a positive constant, $s \sim \mcal{U}[0,1-\tau]$, and $x_s \sim p_s$. We choose a small $\tau$ so that computing $\phi^{\eta}_{s,s+\tau}(x_s)$ with an ODE solver in 1 step is sufficiently accurate. The expression on the RHS equals the expectation of $\| \phi_{s,t}^\theta(x_s) - \phi_{s,t}^\eta(x_s) \|$ scaled by $1/\tau^2$. Note that both IIVM and IAVM are special cases of LFMD in the sense that, if LFMD is $0$, then they would be $0$ as well.

The third term, the \emph{terminal velocity matching} (TVM) loss, aims to enforce the consistency properties in the two lemmas: (b) and (B). Concretely, it tries to ensure consistency between two ways of expressing the velocity of $\phi^\theta_{s,t}(x)$ at the \emph{terminal} time $t$. Using $\tau$ from ITVM, we have that, for any $x$,
\begin{align}
    \frac{\partial}{\partial t}\phi_{s,t}^\theta(x) &\approx \frac{\phi^\theta_{s,t}(x) - \phi^\theta_{s,t-\tau}(x)}{\tau}, \mbox{ and} \label{tvm-term-1} \\
    \frac{\partial}{\partial t}\phi_{s,t}^\theta(x) &\approx \frac{\phi^\theta_{t-\tau,t}( \phi^\theta_{s,t-\tau}(x) ) - \phi^\theta_{s,t-\tau}(x)}{\tau} \label{tvm-term-2} \\
    &= v_{t-\tau,t}^\theta( \phi^\theta_{s,t-\tau}(x)) \label{tvm-term-3}.
\end{align}
Note that \eqref{tvm-term-1} is the numerical differentiation used by PID. Equation \eqref{tvm-term-2} comes from replacing $\phi^\theta_{s,t}(x)$ in \eqref{tvm-term-1} with $\phi^\theta_{t-\tau,t}( \phi^\theta_{s,t-\tau}(x))$. These two quantities should be equal if $\phi^\theta$ satisfies consistency. Equation $\eqref{tvm-term-3}$ follows from the way we formulate $\phi_{s,t}^\theta(x)$. The TVM loss minimizes the squared L2 distance between the right hand sides of \eqref{tvm-term-1} and \eqref{tvm-term-3}, with stop-gradient operator and EMA applied to some of the terms' parameters:
\begin{align*}
    \mcal{L}_{\operatorname{TVM}} &= E_{s,t,x_s} \bigg[ \bigg\| \frac{\phi^\theta_{s,t}(x_s) - \phi^\theta_{s,t-\tau}(x_s)}{\tau} \\ 
    &\quad - v_{t-\tau,t}^{\langle \theta \rangle}( \phi^{[\theta]}_{s,t-\tau}(x_s)) \bigg\|^2 \bigg],    
\end{align*}
where $s \sim \mcal{U}[0,1-\tau]$, $t \sim \mcal{U}[s+\tau,1]$, and $x_s \sim p_s$. Here, $[\theta]$ denotes the stop-gradient version and $\langle \theta \rangle$ the EMA version of the network parameters. Note that the term before the minus sign depends on the parameters under training $\theta$, whereas the term following it does not involve $\theta$ at all, making it a constant with respect to $\theta$. We use EMA to stabilize training, as is done in CM \cite{Song:CM:2023}. We note that, similar to PID, replacing the partial derivative with a finite difference simplifies the implementation of our loss and speeds up computation compared to LFMD.

{\bf Comparison with LFMD and PID.} The form of TVM resembles that of PID, which is in turn a discretized version of LFMD. The main difference is that TVM uses the EMA version of the model under training $v^{\langle \theta \rangle}$ instead of the teacher model $v^\eta$. So, TVM forces the student model to be consistent with itself rather than following the teacher. 

In ITVM, the use of the teacher is relegated to the initial velocity terms. This has two potential benefits. First, in LFMD and PID, the target of velocity matching is $v_t^{\eta}(\phi_{s,t}^\theta(x))$, but the teacher is trained with the CFM loss, which feeds $x_t \sim p_t$ as the teacher's input. Hence, $\phi_{s,t}^\theta(x)$ may be out of the input distribution, leading to degraded target velocities. On the other hand, inputs to the teacher model in IIVM and IAVM are sampled directly from $p_s$ and so are always in distribution. Second, we may modify IIVM and IAVM to use the noisy velocity estimate $v_s(x_s|x_{\data})$ instead of the teacher model's output $v^\eta_s(x_s)$ to derive a loss for TTFM training without a teacher model. However, this investigation is not in the scope of this paper.

{\bf Comparison with CM and SCM.} If we ``equate'' $\phi_{s,t-\tau}^{\theta}(x_s)$ with $\phi_{s,t-\tau}^{[\theta]}(x_s)$, then we can rewrite the RHS of TVM loss\footnote{The rewritten version is not equivalent to the TVM loss because their gradients are different.} as
\begin{align*}
E_{s,t,x_s} \bigg[ \bigg\| \frac{\phi^\theta_{s,t}(x_s) - \phi_{t-\tau,t}^{\langle \theta \rangle}( \phi^{[\theta]}_{s,t-\tau}(x_s))}{\tau} \bigg\|^2 \bigg].
\end{align*}
This suggests that the TVM loss aims to enforce consistency, similar to CM and SCM. There are two key differences between these losses. The first is how they select the intermediate time $u$. TVM chooses $u = t - \tau$, CM chooses $u = s+\tau$, and SCM samples $u$ from $[s,t]$. In other words, TVM follows Lemma~\ref{thm:consistency-extension-1} and Lemma~\ref{thm:consistency-extension-2} more closely. The second is how the teacher model is used. While CM and SCM use the teacher model to push $x_s$ from time $s$ to time $u$, TVM uses the EMA version of the model under training instead, and the teacher model is used only in the initial velocity terms. Note also that IAVM is a special case of CM where the terminal time $t$ is always set to $s+\tau$. Therefore, ITVM is an untangled version of CM where two of its aspects---learning from the teacher and enforcing consistency---are cleanly separated into different loss terms.

{\bf Characterization of trained models.} The ITVM loss does not perfectly enforce the properties in Lemma~\ref{thm:consistency-extension-1} and \ref{thm:consistency-extension-2}. This is because we use a fixed $\tau$ instead of sampling it from $(0,\tau^*]$. Moreover, we also use an ODE solver to compute $\phi^{\eta}_{s,s+\tau}(x_s)$ instead of using the true value. However, we may characterize the behavior of a model trained to completion with ITVM as follows. Assume that the ODE solver is a one-step numerical method such as Euler's, Heun's or Runge--Kutta, and let $\mcal{S}^{\eta,h}_{s,t}(x)$ denote the result using the solver to compute $\phi^\eta_{s,t}(x)$ with step size $h$. 

\begin{theorem}[informal] \label{thm:itvm-correctness}
    Let $\tau = 1/N$ where $N$ is an integer. Let the $\phi_{s,t}^\eta(x_s)$ term in the IAVM loss be computed as $\mcal{S}^{\eta,\tau}_{s,s+\tau}(x)$ (i.e., 1-step integration with step size $\tau$). Suppose that we train $\phi^{\theta}$ to the point that $\langle \theta \rangle = \theta$ and $\mcal{L}_{\operatorname{ITVM}} = 0.$ Under mild assumptions such as continuity of the functions involved, we have that
    \begin{align*}
        \phi_{m\tau, n\tau}^\theta(x) = \mcal{S}_{m\tau, n\tau}^{\eta,\tau}(x),
    \end{align*}
    for all $x \in \Real^d$ and integers $m, n$ such that $0 \leq m \leq n \leq N$.
\end{theorem}

In other words, if training is successful, then the trained model can simulate the ODE solver on multiples of $\tau$ with one function evaluation. A proof can be found in Appendix~\ref{sec:proofs}. Consequently, if the ODE solver is of order $p$, then the error $\| \phi^\eta_{m\tau, n\tau}(x) - \phi^\theta_{m\tau, n\tau}(x) \|$ is $O(\tau^p)$.

\section{Experiments} \label{sec:experiments}

\subsection{Comparison with Other Losses} \label{sec:comparison-with-other-losses}

\begin{figure}
    \scriptsize
    \begin{tabular}{@{\hskip 0.01\linewidth}c@{\hskip 0.01\linewidth}c@{\hskip 0.01\linewidth}c@{\hskip 0.01\linewidth}c@{\hskip 0.01\linewidth}}
        \includegraphics[width=0.23\linewidth]{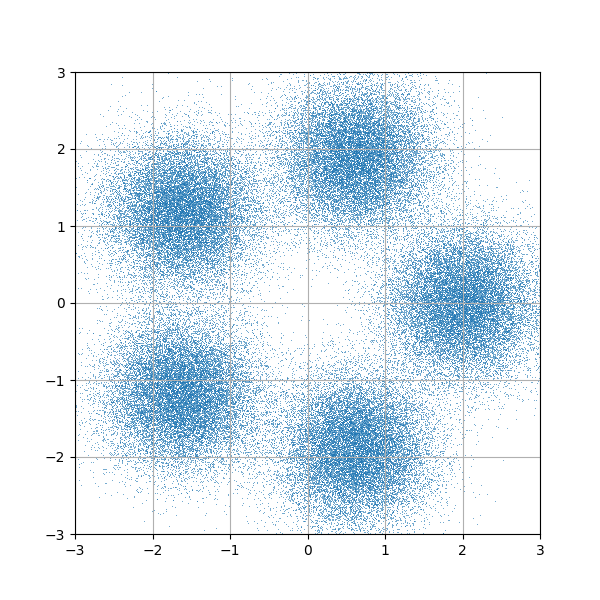} &
        \includegraphics[width=0.23\linewidth]{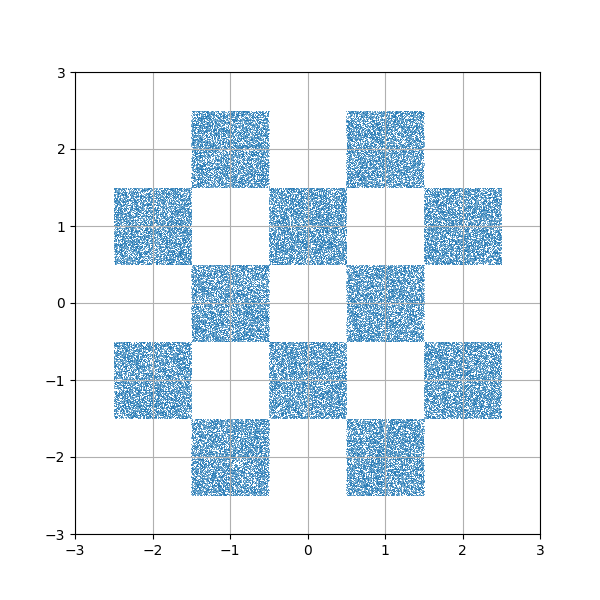} &
        \includegraphics[width=0.23\linewidth]{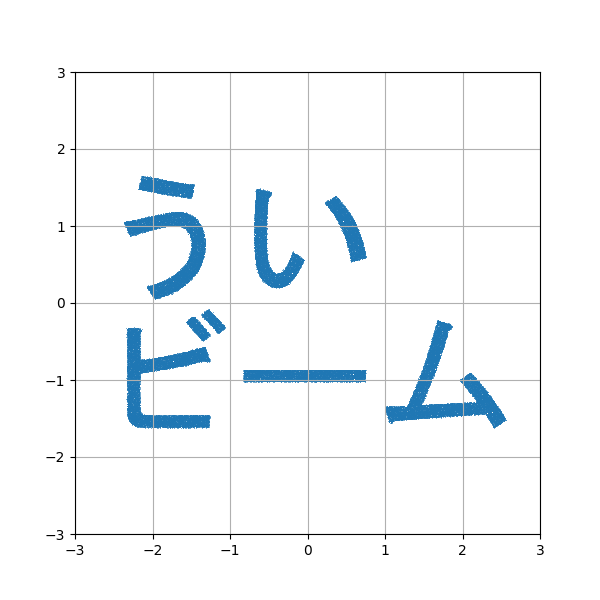} &
        \includegraphics[width=0.23\linewidth]{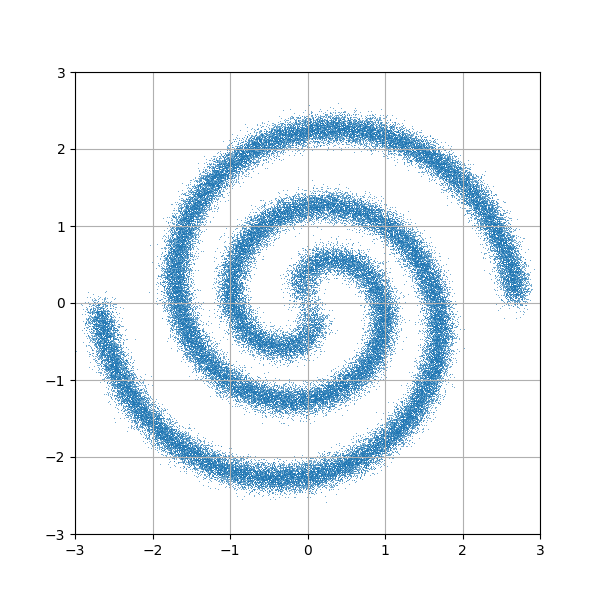} \\
        5LOBES & CHECKER & WORD & SPIRAL
    \end{tabular}    
    \caption{2D datasets. Each has 1M points.}
    \label{fig:2d-datasets}
\end{figure}

We compare ITVM against 3 other losses: EFMD, LFMD, and PID.\footnote{We also experimented with CM and SCM with simple hyperparameter settings but found that the results were worse than most other models we tested. From literature, it seems that one needs to carefully tune the losses' hyperparameters in order for them to be successful \cite{Song:ICM:2023, Kim:CTM:2023}. We decided not to include their results in the paper as we were not able to tune our implementations to the previously reported level. } We use 10 datasets: 4 2D datasets (Figure~\ref{fig:2d-datasets}), 4 tabular datasets, and 2 image datasets. For each dataset, we train a flow matching model and use it as the teacher. For each combination of dataset and loss, we train a number of student models. All student models for the same dataset have the same architecture, but this architecture may be different from that of the teacher's.\footnote{This is natural because the teacher and the student have different function signatures: the teacher's input is of type $\Real \times \Real^d$, but the student's input is of type $\Real \times \Real \times \Real^d$.} Because of this, we initialize the students randomly. 

To simplify experiments, we use $\sigma_{\min} = 0.001$ and $\tau = 0.005$ (obtained through ablation). We use the square of the L2 norm as the distance metric $\mathfrak{d}$ as it applies to all datasets. Times are floating point numbers, and they are sampled uniformly from their respective intervals. We use one step of Heun's method to evaluate the implicit flow in the ITVM loss. We vary only the EMA decay rate $\mu$, which can be either $0, 0.9, 0.99,$ or $0.999$. As a result, for ITVM, there are at most 4 student models. Our experimental setup allows us to compare the losses directly because, if two student models require the same hyperparameters (for example, $\tau$ in case of PID and ITVM), their values would be set identically. Implementation details, including network architectures and training processes, are in Appendix~\ref{sec:comparison-setup-and-data}.

We evaluate each student by assessing its ability to generate samples with few function evaluations. We use it to generate four sets of samples from $p_1 = p_{\data} * \mcal{N}(0,\sigma_{\min}^2I)$. All sets have the same size, and they correspond to sampling with 1, 2, 4, and 8 function evaluations. For each set, we compute a metric, resulting in 4 numbers per student model. The metric used depends on the dataset. For 2D and tabular datasets, we use the samples to stochastically estimate the KL divergence between the student's and teacher's distributions. Details of this process can be found in Appendix~\ref{sec:metrics}. For image datasets, we use the Fr\'{e}chet Inception Distance (FID) \cite{Huesel:2017} computed with the code from Karras \etal's EDM paper \yrcite{Karras:2022}.

Since each dataset has up to seven student models, we condense their performance statistics for conciseness and clarity.
For the ITVM loss, we rank its students according to the 4 scores, resulting in 4 ranks per student. Then, we fuse the ranks with the reciprocal rank fusion algorithm (RRF) \cite{Cormack:2009} with $k = 60$ and show the model with the best fused rank. Full tables and generated samples can be found in Appendix~\ref{sec:comparison-setup-and-data}.

\begin{table}
\caption{Comparison of performance of student models for 2D datasets, showing {\color{Green} 1st}, {\color{Green} 2nd}, and {\color{Blue} 3rd} best scores.}\label{tab:compare-2d-students} \vspace{0.5em}
{\fontsize{7.5pt}{9pt}\selectfont
    \centering
    \begin{tabular}{lrrrr}
        \toprule
        \textbf{Loss}
        & \multicolumn{1}{c}{\textbf{NFE = 1}}
        & \multicolumn{1}{c}{\textbf{NFE = 2}}
        & \multicolumn{1}{c}{\textbf{NFE = 4}}
        & \multicolumn{1}{c}{\textbf{NFE = 8}} \\
\hline
\multicolumn{1}{l}{\mc \textbf{Dataset}: 5LOBES} & \multicolumn{4}{c}{\mc \textbf{Metric}: KL divergence 
$\times 1000$ ($\downarrow$)} \\
EFMD & {\color{Red} 0.1520}& {\color{Green} 0.1410}& {\color{Blue} 0.1300}& {\color{Blue} 0.1220}\\
LFMD & {\color{Green} 0.1800}& {\color{Blue} 0.1430}& {\color{Green} 0.1210}& {\color{Green} 0.1060}\\
PID & 0.3300& 0.3380& 0.3060& 0.2800\\
ITVM ($\mu = 0.99$) & {\color{Blue} 0.1870}& {\color{Red} 0.1380}& {\color{Red} 0.1070}& {\color{Red} 0.0780}\\
\hline
\multicolumn{1}{l}{\mc \textbf{Dataset}: CHECKER} & \multicolumn{4}{c}{\mc \textbf{Metric}: KL divergence $\times 10$ ($\downarrow$)} \\ 
EFMD & 0.6994& 0.5629& 0.3344& 0.1935\\
LFMD & {\color{Blue} 0.2737}& {\color{Blue} 0.1560}& {\color{Blue} 0.1206}& 0.1046\\
PID & {\color{Green} 0.2512}& {\color{Green} 0.1459}& {\color{Green} 0.1084}& {\color{Blue} 0.0898}\\
ITVM ($\mu = 0.9$) & {\color{Red} 0.1856}& {\color{Red} 0.0903}& {\color{Red} 0.0852}& {\color{Red} 0.0713}\\
\hline
\multicolumn{1}{l}{\mc \textbf{Dataset}: WORD} & \multicolumn{4}{c}{\mc \textbf{Metric}: KL divergence ($\downarrow$)}  \\
EFMD & 0.4795& 0.3556& 0.2312& 0.1511\\
LFMD & 0.7343& {\color{Blue} 0.1672}& {\color{Blue} 0.1282}& 0.0930\\
PID & {\color{Green} 0.4191}& {\color{Green} 0.1608}& {\color{Green} 0.1186}& {\color{Green} 0.0763}\\
ITVM ($\mu = 0.9$) & {\color{Red} 0.2475}& {\color{Red} 0.1305}& {\color{Red} 0.0962}& {\color{Blue} 0.0850}\\
\hline
\multicolumn{1}{l}{\mc \textbf{Dataset}: SPIRAL} & \multicolumn{4}{c}{\mc \textbf{Metric}: KL divergence $\times 10$ ($\downarrow$)}  \\
EFMD & 1.4024& 0.9964& 0.6922& 0.4299\\
LFMD & {\color{Green} 0.4254}& {\color{Blue} 0.2898}& {\color{Blue} 0.3064}& 0.2345\\
PID & {\color{Blue} 0.4830}& {\color{Green} 0.2309}& {\color{Green} 0.2547}& 0.1912\\
ITVM ($\mu = 0.999$) & {\color{Red} 0.3407}& {\color{Red} 0.1560}& {\color{Red} 0.1337}& {\color{Red} 0.0990}\\
\bottomrule
    \end{tabular}
    }
    \vspace{-2ex}
\end{table}

{\bf 2D datasets.} The teacher model and the student models are MLPs having ELU \cite{Clevert:2016} activation functions and 8 hidden layers. Each hidden layer of the teacher model has 512 units, and the student 1,024. The performance data of the student models can be found in Table~\ref{tab:compare-2d-students}.
We observe that the ITVM loss consistently achieves the best scores.

\begin{table}[]
\caption{Comparison of performance of student models for tabular datasets, showing {\color{Red} 1st}, {\color{Green} 2nd}, and {\color{Blue} 3rd} best scores.}
    \label{tab:compare-tabular-students}\vspace{0.5em}
    {\fontsize{7.2pt}{9pt}\selectfont
    \centering
    \begin{tabular}{lrrrr}
        \toprule
        \textbf{Loss}
        & \multicolumn{1}{c}{\textbf{NFE = 1}}
        & \multicolumn{1}{c}{\textbf{NFE = 2}}
        & \multicolumn{1}{c}{\textbf{NFE = 4}}
        & \multicolumn{1}{c}{\textbf{NFE = 8}} \\
\hline
\multicolumn{1}{l}{\mc \textbf{Dataset}: GAS} & \multicolumn{4}{c}{\mc \textbf{Metric}: KL divergence 
 ($\downarrow$)} \\
EFMD & 19.034& 27.090& 37.149& 41.112\\
LFMD & {\color{Blue} 3.482}& {\color{Green} 2.194}& {\color{Green} 1.458}& {\color{Green} 0.836}\\
PID & {\color{Green} 2.864}& {\color{Blue} 3.153}& {\color{Blue} 1.727}& {\color{Blue} 1.017}\\
ITVM ($\mu = 0.99$) & {\color{Red} 2.577}& {\color{Red} 1.944}& {\color{Red} 1.300}& {\color{Red} 0.731}\\
\hline 
\multicolumn{1}{l}{\mc \textbf{Dataset}: HEPMASS} & \multicolumn{4}{c}{\mc \textbf{Metric}: KL divergence ($\downarrow$)} \\ 
EFMD & 3.8718& 3.3294& 2.5417& 2.1503\\
LFMD & {\color{Green} 1.1869}& {\color{Green} 0.9390}& {\color{Red} 0.3589}& {\color{Green} 0.2128}\\
PID & {\color{Blue} 1.5424}& {\color{Red} 0.8343}& {\color{Blue} 0.3742}& {\color{Blue} 0.2176}\\
ITVM ($\mu = 0.99$) & {\color{Red} 1.1232}& {\color{Blue} 0.9781}& {\color{Green} 0.3605}& {\color{Red} 0.1987}\\
\hline 
\multicolumn{1}{l}{\mc \textbf{Dataset}: MINIBOONE} & \multicolumn{4}{c}{\mc \textbf{Metric}: KL divergence ($\downarrow$)}  \\
EFMD & 68.0274& 155.2350& 191.7971& 206.9342\\
LFMD & {\color{Green} 7.9273}& {\color{Green} 5.9040}& {\color{Red} 5.5466}& {\color{Red} 5.5661}\\
PID & {\color{Red} 7.9176}& {\color{Red} 5.7900}& {\color{Green} 5.5739}& {\color{Green} 6.0128}\\
ITVM ($\mu = 0.9$) & {\color{Blue} 10.0666}& {\color{Blue} 8.5327}& {\color{Blue} 7.4486}& {\color{Blue} 6.7073}\\
\hline 
\multicolumn{1}{l}{\mc \textbf{Dataset}: POWER} & \multicolumn{4}{c}{\mc \textbf{Metric}: KL divergence ($\downarrow$)}  \\
EFMD & 8.1597& 9.4152& 11.8775& 13.1891\\
LFMD & {\color{Blue} 0.2094}& {\color{Blue} 0.1018}& {\color{Green} 0.0377}& {\color{Green} 0.0205}\\
PID & {\color{Green} 0.1731}& {\color{Green} 0.0906}& {\color{Blue} 0.0485}& {\color{Blue} 0.0271}\\
ITVM ($\mu = 0.99$) & {\color{Red} 0.1160}& {\color{Red} 0.0496}& {\color{Red} 0.0241}& {\color{Red} 0.0135}\\
\bottomrule 
    \end{tabular}
    }
    \vspace{-2ex}
\end{table}

{\bf Tabular datasets.} We use 4 datasets of multi-dimensional vectors previously used by Papamakarios \etal~\yrcite{Papamakarios:2017}: GAS (8D, $n\approx1$M), HEPMASS (21D, $n\approx350$K), MINIBOONE (42D, $n \approx 33$K), and POWER (6D, $n \approx 1.8$M). Teacher and student models are MLPs with skip connections \cite{Preechakul:2022}, which we found to perform better than standard MLPs. Metrics can be found in Table~\ref{tab:compare-tabular-students}. Again, we observe that the ITVM loss yields the best scores in most cases except for MINIBOONE, where it ranks 3rd in all metrics after LFMD and PID.

\begin{table}[]
\caption{Comparison of performance of student models for image datasets, showing {\color{Red} 1st}, {\color{Green} 2nd}, and {\color{Blue} 3rd} best scores.}
    \label{tab:compare-image-students}\vspace{0.5em}
    \centering
    {\fontsize{7.7pt}{9pt}\selectfont
    \centering
    \begin{tabular}{lrrrr}
         \toprule
        \textbf{Loss}
        & \multicolumn{1}{c}{\textbf{NFE = 1}}
        & \multicolumn{1}{c}{\textbf{NFE = 2}}
        & \multicolumn{1}{c}{\textbf{NFE = 4}}
        & \multicolumn{1}{c}{\textbf{NFE = 8}} \\
\hline
\multicolumn{1}{l}{\mc \textbf{Dataset}: MNIST} & \multicolumn{4}{c}{\mc \textbf{Metric}: FID 
 ($\downarrow$) \quad Teacher's FID = 0.309} \\
EFMD & 17.8660& 16.0118& 16.1313& 18.3542\\
LFMD & {\color{Blue} 12.2794}& {\color{Blue} 3.5004}& {\color{Blue} 1.1710}& {\color{Blue} 0.8695}\\
PID & {\color{Green} 2.3384}& {\color{Green} 1.7580}& {\color{Green} 0.8422}& {\color{Red} 0.5864}\\
ITVM ($\mu = 0.99$) & {\color{Red} 2.1745}& {\color{Red} 1.1670}& {\color{Red} 0.8123}& {\color{Green} 0.6039}\\
\hline 
\multicolumn{1}{l}{\mc \textbf{Dataset}: CIFAR-10} & \multicolumn{4}{c}{\mc \textbf{Metric}: FID ($\downarrow$) \quad Teacher's FID = 3.798}  \\
EFMD & 57.0296& 42.5501& 29.9066& 26.7983\\
LFMD & {\color{Blue} 12.8200}& {\color{Green} 6.9392}& {\color{Green} 4.7665}& {\color{Red} 4.1941}\\
PID & {\color{Green} 12.7306}& {\color{Blue} 7.1180}& {\color{Blue} 5.0251}& {\color{Blue} 4.5550}\\
ITVM ($\mu = 0.999$) & {\color{Red} 9.5318}& {\color{Red} 6.2173}& {\color{Red} 4.7510}& {\color{Green} 4.4079}\\
\bottomrule
    \end{tabular}
    }
    \vspace{-2ex}
\end{table}

{\bf Image datasets.} We experimented on MNIST and CIFAR-10. Teacher and student models are U-Nets with attention \cite{Ho:2020}. Using the Euler's method, the teacher models achieve FID scores of 0.309 for MNIST at 100 NFEs and 3.798 for CIFAR-10 at 1,000 NFEs. Most student models were trained under the same batch size. However, for the CIFAR-10 dataset, the GPUs available to us lacked sufficient RAM to train the EFMD and LFMD models, requiring us to reduce the batch size from 80 to 56. 

The FID scores of the student models are reported in Table~\ref{tab:compare-image-students}
We see that ITVM performs better than other losses in most cases. However, the scores for CIFAR-10 that our implementations of ITVM and other losses achieve are not as good as previously reported in other papers: for 1-step generation, CM yields FID of 2.51 \cite{Song:ICM:2023}, SCM 1.73 \cite{Kim:CTM:2023}, PID 3.68 \cite{Tee:2024:PID}. We attribute this discrepancy to several factors. First, we surmise that the noise schedule of our models is not as optimized for image datasets as those of diffusion models, such as those by Karras \etal~\yrcite{Karras:2022}. Second, our CIFAR-10 teacher model does not perform as well as those used by other papers, and we cannot use those models as teacher due to difference in formulation (flow matching vs standard diffusion). Third, we have not extensively tuned network architecture, training process, and hyperparameters. Indeed, to achieve competitive FID scores, CM requires carefully crafted time discretization curriculum, noise schedule, and distance function \cite{Song:ICM:2023}, CTM requires two more auxiliary losses \cite{Kim:CTM:2023}, and PID as implemented in Tee \etal's paper uses LPIPS rather than L2 as the distance metric \yrcite{Tee:2024:PID}. Still, our implementation of LFMD performs better than what is reported in Boffi~\etal's paper: 6.94 vs 7.63 for 2-step generation, and 4.77 vs 6.04 for 4-step generation.

All in all, results in this section show that, when other settings are controlled, ITVM performs better than baselines on multiple dataset types and model architectures. Although SOTA performance was not achieved on CIFAR-10, we believe ITVM has the potential to produce one with a more optimized implementation.

\begin{table}[]
\caption{An ablation study on the two initial velocity terms of the ITVM loss. The EMA decay rates for the WORD and MNIST datasets are $0.99$, and the CIFAR-10 dataset $0.999$. Colors indicate {\color{Red} 1st}, {\color{Green} 2nd}, and {\color{Blue} 3rd} best scores within the table.}
    \label{tab:iivm-iiam-ablation}\vspace{0.5em}
    \centering
    {\fontsize{7.6pt}{9pt}\selectfont
    \centering
    \begin{tabular}{lrrrr}
        \toprule
        \textbf{Loss}
        & \multicolumn{1}{c}{\textbf{NFE = 1}}
        & \multicolumn{1}{c}{\textbf{NFE = 2}}
        & \multicolumn{1}{c}{\textbf{NFE = 4}}
        & \multicolumn{1}{c}{\textbf{NFE = 8}} \\
\hline
\multicolumn{1}{l}{\mc \textbf{Dataset}: WORD} & \multicolumn{4}{c}{\mc \textbf{Metric}: KL divergence
 ($\downarrow$)} \\
LFMD & 0.7343& 0.1672& {\color{Blue} 0.1282}& {\color{Green} 0.0930}\\
PID & 0.4191& 0.1608& {\color{Green} 0.1186}& {\color{Red} 0.0763}\\
\arrayrulecolor{grayline}\hline
\noalign{\vskip 1pt}
ITVM: IIVM only & {\color{Green} 0.2482}& {\color{Green} 0.1480}& 0.1405& 0.1384\\
ITVM: IAVM only & {\color{Red} 0.2454}& {\color{Blue} 0.1483}& 0.1392& 0.1380\\
ITVM: IIVM+IAVM & {\color{Blue} 0.2646}& {\color{Red} 0.1298}& {\color{Red} 0.1111}& {\color{Blue} 0.1073}\\
\arrayrulecolor{black}\hline
\multicolumn{1}{l}{\mc \textbf{Dataset}: HEPMASS} & \multicolumn{4}{c}{\mc \textbf{Metric}: KL divergence
 ($\downarrow$)} \\
LFMD & {\color{Blue} 1.1869}& {\color{Green} 0.9390}& {\color{Red} 0.3589}& {\color{Green} 0.2128}\\
PID & 1.5424& {\color{Red} 0.8343}& {\color{Blue} 0.3742}& {\color{Blue} 0.2176}\\
\arrayrulecolor{grayline}\hline
\noalign{\vskip 1pt}
ITVM: IIVM only & 1.2296& {\color{Blue} 0.9514}& 0.3989& 0.2715\\
ITVM: IAVM only & {\color{Green} 1.1240}& 1.2420& 0.4003& 0.2718\\
ITVM: IIVM+IAVM & {\color{Red} 1.1232}& 0.9781& {\color{Green} 0.3605}& {\color{Red} 0.1987}\\
\arrayrulecolor{black}\hline
\multicolumn{1}{l}{\mc \textbf{Dataset}: MNIST} & \multicolumn{4}{c}{\mc \textbf{Metric}: FID 
 ($\downarrow$) \quad Teacher's FID = 0.309} \\
LFMD & 12.2794& 3.5004& 1.1710& 0.8695\\
PID & {\color{Green} 2.3384}& {\color{Blue} 1.7580}& {\color{Blue} 0.8422}& {\color{Red} 0.5864}\\
\arrayrulecolor{grayline}\hline
\noalign{\vskip 1pt}
ITVM: IIVM only & 2.5118& 2.7000& 2.3260& 1.9718\\
ITVM: IAVM only & {\color{Blue} 2.4394}& {\color{Green} 1.2965}& {\color{Green} 0.8344}& {\color{Blue} 0.6548}\\
ITVM: IIVM+IAVM & {\color{Red} 2.1745}& {\color{Red} 1.1670}& {\color{Red} 0.8123}& {\color{Green} 0.6039}\\
\arrayrulecolor{black}\hline
\multicolumn{1}{l}{\mc \textbf{Dataset}: CIFAR} & \multicolumn{4}{c}{\mc \textbf{Metric}: FID 
 ($\downarrow$) \quad Teacher's FID = 3.798} \\
LFMD & 12.8200& 6.9392& {\color{Blue} 4.7665}& {\color{Red} 4.1941}\\
PID & 12.7306& 7.1180& 5.0251& 4.5550\\
\arrayrulecolor{grayline}\hline
\noalign{\vskip 1pt}
ITVM: IIVM only & {\color{Blue} 9.7885}& {\color{Blue} 6.7481}& 5.4656& 5.0921\\
ITVM: IAVM only & {\color{Red} 8.4243}& {\color{Red} 5.7748}& {\color{Red} 4.6409}& {\color{Green} 4.2399}\\
ITVM: IIVM+IAVM & {\color{Green} 9.5318}& {\color{Green} 6.2173}& {\color{Green} 4.7510}& {\color{Blue} 4.4079}\\
\arrayrulecolor{black}\bottomrule
    \end{tabular}
    }
    \vspace{-2ex}
\end{table}

\begin{figure}
    \scriptsize
    \centering
    \begin{tabular}{@{\hskip 0.01\linewidth}c@{\hskip 0.01\linewidth}c@{\hskip 0.01\linewidth}}
        \includegraphics[width=0.5\linewidth]{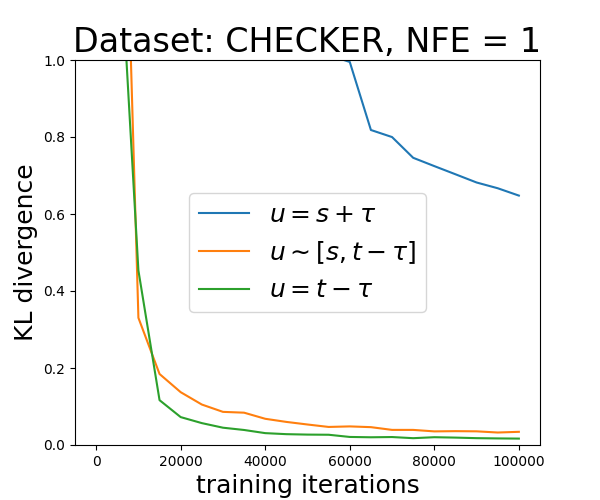} &
        \includegraphics[width=0.5\linewidth]{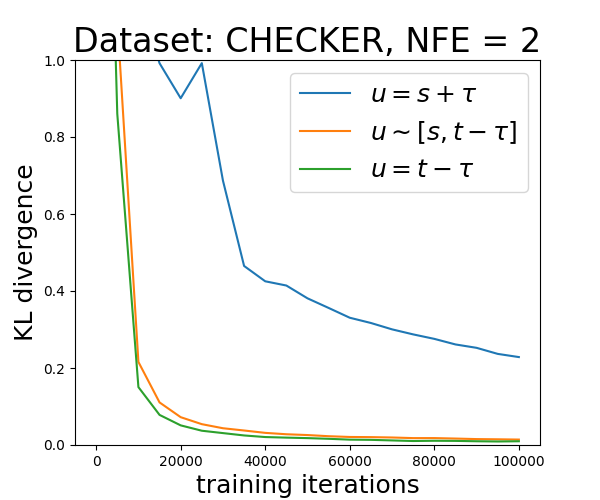} \\
        \includegraphics[width=0.5\linewidth]{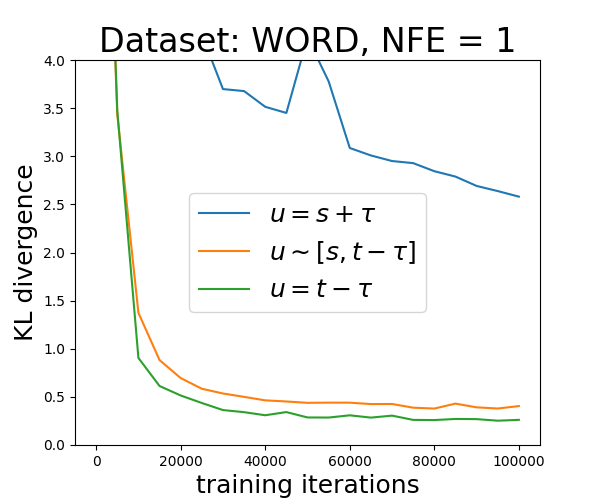} &
        \includegraphics[width=0.5\linewidth]{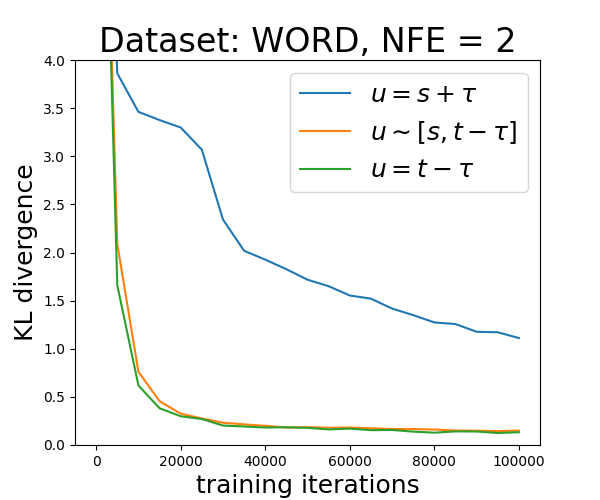} 
    \end{tabular}    
    \vspace{-0.7em}
    \caption{The effect of how the intermediate time $u$ is chosen on the evolution of the KL divergence between distributions of student models and teacher models during training. For the strategy that samples from $[s,t]$, we sample from $[s,t-\tau]$ instead to avoid the denominator in the TVM loss being too small.}
    \label{fig:u-ablation}
\end{figure}

\subsection{Analysis of Design Choices and Hyperparameters}

{\bf Initial velocity terms.} The ITVM loss has two initial velocity terms: IIVM and IAVM. However, their motivating insights, Lemma~\ref{thm:consistency-extension-1} and Lemma~\ref{thm:consistency-extension-2}, are completely independent. As such, one may think that only one term should be enough. This is mostly true. In Table~\ref{tab:iivm-iiam-ablation}, we show the performance of student models on four datasets. They were trained with TVM and either one of or both initial velocity terms. We also include LFMD and PID as baselines. We can see that all versions of IIVM are better or competitive with the baselines, especially on 1- and 2-step generations. However, for the MNIST and CIFAR-10 datasets, not using IAVM degrades performance on 4- and 8-step generations, implying that the term is important for FID when NFE is higher. IIVM, on the other hand, seems to have less impact: using it together with IAVM tends to improve overall performance, but not significantly. Moreover, for CIFAR-10, IAVM alone beats the full loss by about 1 point on 1-step and 2-steps generation. As a result, while using both terms is likely to yield the best student model, one may drop IIVM if training budget is constrained.

{\bf Intermediate time in TVM.} As noted in Section~\ref{sec:method}, the TVM term enforces a version of the consistency property that fixes the intermediate time $u$  to be $t-\tau$. We present evidence that this strategy of picking $u$ is better than those employed by CM (picking $u = s + \tau$) and SCM (sampling $u$ from $[s,t]$) in Figure~\ref{fig:u-ablation}. We observe that, for 1-step generation, ITVM's strategy (green curves) consistently performs the best, and there are visible gaps with the second best strategy. Picking $u = s+\tau$, on the other hand, leads to significantly worse performance. For 2-step generation, the gap between $u=t-\tau$ and $u \sim [s,t]$ narrows down, while $u = s+\tau$ still performs much worse. Plots for other NFEs are included in Appendix~\ref{sec:full-u-ablation}. We conclude that, compared to the other two strategies, picking $u = t-\tau$ results in better 1-step generation performance and the same level of performance at higher NFEs.

{\bf Step size $\tau$.} To study the effect of $\tau$, we trained student TTFMs for the CHECKER and WORD datasets, setting $\tau$ to $0.1$, $0.05$, $0.01$, $0.005$, and $0.001$. 
Figure~\ref{fig:tau-ablation} shows the KL divergence between the student's and teacher's $p_1$ distributions during training for NFE = 1 and 2.
Most $\tau$ values yield similar KL curves and final performance, except the highest $\tau = 0.1$ (blue curves), which performs noticeably worse in three plots.
As a result, we think it suffices to pick a $\tau$ value that is low enough, and we recommend $0.005$. Plots for other NFEs are available in Appendix~\ref{sec:ful-tau-ablation}.

\begin{figure}[]
    \scriptsize
    \centering
    \begin{tabular}{@{\hskip 0.01\linewidth}c@{\hskip 0.01\linewidth}c@{\hskip 0.01\linewidth}}
        \includegraphics[width=0.5\linewidth]{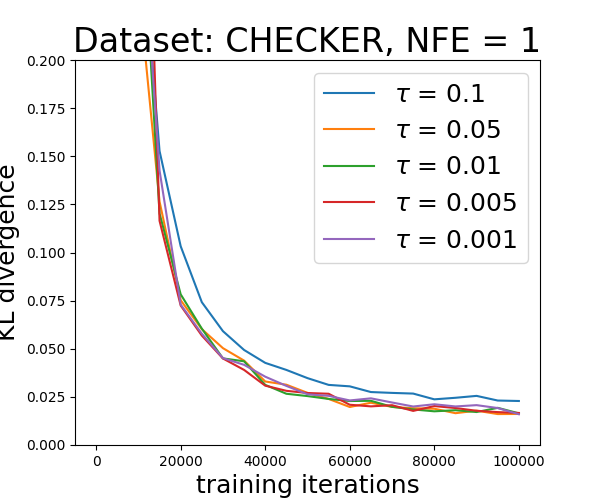} &
        \includegraphics[width=0.5\linewidth]{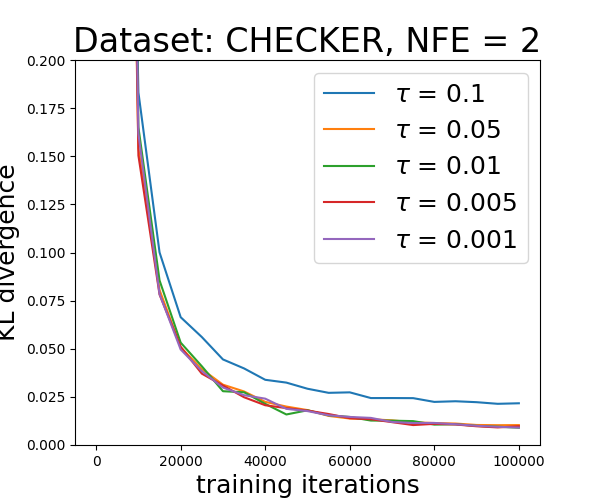} \\
        \includegraphics[width=0.5\linewidth]{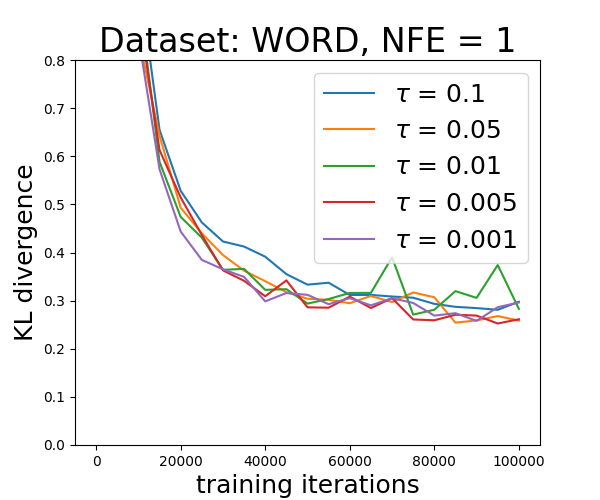} &
        \includegraphics[width=0.5\linewidth]{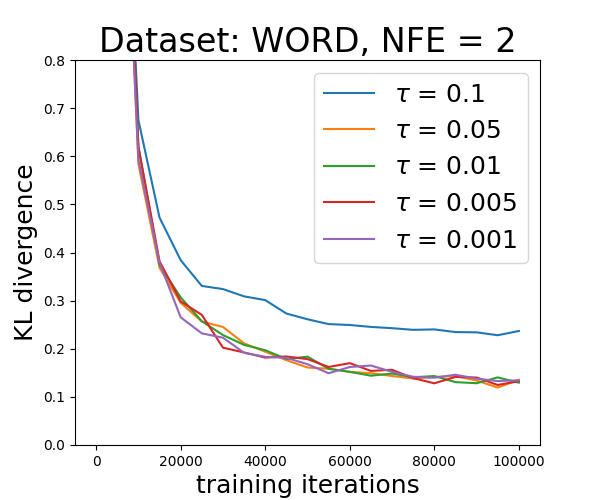} 
    \end{tabular}    
    \caption{The effect of the hyperparameter $\tau$ on the evolution of the KL divergence between distributions of student models and teacher models during training.}
    \label{fig:tau-ablation}
\end{figure}

{\bf EMA decay rate $\mu$.} As stated in Section~\ref{sec:comparison-with-other-losses}, when we trained student TTFMs, we set the EMA decay rate $\mu$ to $0$, $0.9$, $0.99$, and $0.999$ and trained a model for each value. The metrics for these models are available in Appendix~\ref{sec:comparison-setup-and-data}, but we reproduce them for some datasets in Table~\ref{tab:mu-ablation} for easier reference. We observe that, while the optimal $\mu$ varies from dataset to dataset, $\mu = 0$ (i.e., using no EMA at all) tends to rank lower, but $\mu = 0.99$ tends to rank higher. As a result, we recommend using $\mu = 0.99$ and avoiding $\mu = 0$.

\begin{table}[]
\caption{Effects of the EMA decay parameter $\mu$ on student TTFM performance for 4 datasets. The rightmost column shows fused ranks computed with RRF \cite{Cormack:2009}, using $k = 60$.}
    \label{tab:mu-ablation}\vspace{0.5em}
    \centering
    {\fontsize{7.6pt}{9pt}\selectfont
    \begin{tabular}{lrrrrc}
        \toprule
        \textbf{Loss}
        & \multicolumn{1}{c}{\textbf{NFE = 1}}
        & \multicolumn{1}{c}{\textbf{NFE = 2}} 
        & \multicolumn{1}{c}{\textbf{NFE = 4}}
        & \multicolumn{1}{c}{\textbf{NFE = 8}}
        & \multicolumn{1}{c}{\textbf{Rank}} \\
\hline
\multicolumn{2}{l}{\mc \textbf{Dataset}: WORD} & \multicolumn{4}{c}{\mc \textbf{Metric}: KL divergence 
 ($\downarrow$)} \\
$\mu = 0.0$ & 0.5638& 0.3561& 0.3185& 0.2920& 4\\
$\mu = 0.9$ & {\color{Red} 0.2475}& {\color{Green} 0.1305}& {\color{Red} 0.0962}& {\color{Red} 0.0850}& {\color{Red} 1}\\
$\mu = 0.99$ & {\color{Green} 0.2646}& {\color{Red} 0.1298}& {\color{Green} 0.1111}& {\color{Green} 0.1073}& {\color{Green} 2}\\
$\mu = 0.999$ & {\color{Blue} 0.2677}& {\color{Blue} 0.1383}& {\color{Blue} 0.1190}& {\color{Blue} 0.1088}& {\color{Blue} 3}\\
\hline
\multicolumn{2}{l}{\mc \textbf{Dataset}: HEPMASS} & \multicolumn{4}{c}{\mc \textbf{Metric}: KL divergence 
 ($\downarrow$)} \\
$\mu = 0.0$ & 1.3484& {\color{Blue} 1.0173}& 0.4186& 0.2351& 4\\
$\mu = 0.9$ & {\color{Red} 1.0920}& 1.0249& {\color{Blue} 0.3747}& {\color{Blue} 0.2056}& {\color{Blue} 3}\\
$\mu = 0.99$ & {\color{Green} 1.1232}& {\color{Red} 0.9781}& {\color{Red} 0.3605}& {\color{Green} 0.1987}& {\color{Red} 1}\\
$\mu = 0.999$ & {\color{Blue} 1.1846}& {\color{Green} 0.9983}& {\color{Green} 0.3621}& {\color{Red} 0.1897}& {\color{Green} 2}\\
\hline
\multicolumn{2}{l}{\mc \textbf{Dataset}: MNIST} & \multicolumn{4}{c}{\mc \textbf{Metric}: FID 
 ($\downarrow$)} \\
 $\mu = 0.0$ & 2.6642& 1.4715& {\color{Green} 0.8203}& {\color{Green} 0.6232}& {\color{Blue} 3}\\
$\mu = 0.9$ & {\color{Blue} 2.5324}& {\color{Blue} 1.4619}& 0.9594& {\color{Blue} 0.6377}& 4\\
$\mu = 0.99$ & {\color{Red} 2.1745}& {\color{Red} 1.1670}& {\color{Red} 0.8123}& {\color{Red} 0.6039}& {\color{Red} 1}\\
$\mu = 0.999$ & {\color{Green} 2.3011}& {\color{Green} 1.1849}& {\color{Blue} 0.8691}& 0.6705& {\color{Green} 2}\\
\hline
\multicolumn{2}{l}{\mc \textbf{Dataset}: CIFAR-10} & \multicolumn{4}{c}{\mc \textbf{Metric}: FID ($\downarrow$) } \\
$\mu = 0.0$ & 9.9793& 6.9036& {\color{Blue} 5.1475}& 4.6902& 4\\
$\mu = 0.9$ & {\color{Red} 9.3762}& {\color{Blue} 6.5876}& 5.1617& {\color{Blue} 4.6007}& {\color{Blue} 3}\\
$\mu = 0.99$ & {\color{Blue} 9.7344}& {\color{Green} 6.2690}& {\color{Green} 4.8295}& {\color{Green} 4.4825}& {\color{Green} 2}\\
$\mu = 0.999$ & {\color{Green} 9.5318}& {\color{Red} 6.2173}& {\color{Red} 4.7510}& {\color{Red} 4.4079}& {\color{Red} 1}\\
\bottomrule
    \end{tabular}
    }
\end{table}
\vspace{-1em}

\section{Conclusion}

We present initial/terminal velocity matching (ITVM), a new loss function for distilling a flow matching model to a two-time flow model (TTFM) to allow fast and flexible sample generation. Our loss is motivated by the fact that one can ensure that the network satisfies the Lagrangian PDE by forcing it (1) to copy the teacher's model velocities at the initial time $s$ and (2) to be consistent with itself near the terminal time $t$. The loss terms include two that are special cases of LFMD, and one that rewrites it by replacing the teacher model with the EMA version of the model under training. Experiments using simple hyperparameter settings show that ITVM outperforms several baselines including LFMD and that its design choices, including using two initial terms and the strategy of picking the intermediate time in the terminal term, have merit. 
With systematic and dataset-specific tuning of key hyperparameters---the noise schedule, time discretization, loss term weights, distance function, and training curriculum---we anticipate ITVM to achieve competitive performance on heavily optimized benchmarks like CIFAR-10 and ImageNet.
We believe our exploration contributes new insights to the problem of distillation of diffusion-based generative models.


\bibliography{main}
\bibliographystyle{icml2025}

\newpage
\appendix
\onecolumn

\section{Proofs} \label{sec:proofs}

We use the same notations as those in Section~\ref{sec:preliminary}. We repeat some definitions  here for easy reference. Let $v^{\eta}: \Real \times \Real^d \ra \Real^d$ denote the flow matching model with parameters $\eta$ that serves as the teacher. The implicit flow $\phi^\eta_{s,t}$ is defined as follows.
\begin{align*}
\phi_{s,s}^\eta(x) &= x, \mbox{ and }, \\
\phi^\eta_{s,t}(x) &= x + \int_{s}^t v^{\eta}(\phi_{s,u}^\eta(x))\, \dee u\mbox{ for } s < t.
\end{align*}
We note that $\phi^{\eta}$ satisfies the Lagrangian PDE, and this can be shown by just differentiating the second equation with respect to $t$. Let $v^\theta: \Real \times \Real \times \Real^d \ra \Real^d$ denote the average velocity model with parameters $\theta$. Lastly, let $\phi_{s,t}^\theta(x) = x + (t-s)v^{\theta}_{s,t}(x)$ denote a TTFM.

\begin{lemma} 
Suppose there exists $\tau^* > 0$ such that---for all $x \in \Real^d$, $0 \leq s < t \leq 1$, and $\tau < \min\{ \tau^*, t-s \}$---it is true that
\begin{enumerate}
    \item[(a)] $\phi^\theta_{s,s+\tau}(x) = \phi^{\eta}_{s,s+\tau}(x)$.
    \item[(b)] $\phi^\theta_{s,t}(x) = \phi^\theta_{t-\tau,t}(\phi^\theta_{s,t-\tau}(x))$.
\end{enumerate}
Then, $\phi^{\theta}_{s,t}(x) = \phi^{\eta}_{s,t}(x)$ for all $x$ and $0 \leq s \leq t \leq 1$.
\end{lemma}

\begin{proof}
If $s = t$, then $\phi_{s,t}^\theta(x) = x = \phi_{s,t}^\eta(x).$

Assume that $s < t$. Let $N$ be a positive integer such that $\tau^* > (t-s)/N$. For $n = 0, 1, \dotsc, N$, let $u_n = s + n(t-s)/N$. We have that $\phi_{s,t}^\theta(x) =  \phi_{u_0, u_N}^{\theta}(x)$. According to Property (b), we have that
\begin{align*}
    \phi_{u_0, u_n}^{\theta}(x) = \phi_{u_{n-1}, u_n}^{\theta}(\phi^\theta_{u_0, u_{n-1}}(x))
\end{align*}
for any $1 \leq n \leq N$ because $u_{n} - u_{n-1} = (t-s)/N < \tau^*$. Using the above equation, we have that
\begin{align*}
    \phi_{u_0, u_N}^{\theta}(x) = \phi_{u_{N-1}, u_N}^{\theta}(\phi^\theta_{u_{N-2}, u_{N-2}}(\dotsb \phi^\theta_{u_1,u_2}(\phi^{\theta}_{u_0,u_1}(x)) \dotsb))
\end{align*}
Using Property (a), we can change each application of $\phi^\theta$ on the right hand side to $\phi^{\eta}$ because, again, $u_n-u_{n-1}<\tau^*$.
\begin{align*}
    \phi_{u_0, u_N}^{\theta}(x) &= \phi_{u_{N-1}, u_N}^{\eta}(\phi^\eta_{u_{N-2}, u_{N-2}}(\dotsb \phi^\eta_{u_1,u_2}(\phi^{\eta}_{u_0,u_1}(x)) \dotsb)) \\
    &= \phi_{u_0, u_N}^{\eta}(x).
\end{align*}
The last line follows from the fact that $\phi^\eta$ automatically satisfies the consistency property.
\end{proof}

\begin{lemma}
    Suppose that $v^\theta$ and $\phi^\theta$ satisfies the following properties.
    \begin{enumerate}
        \item[(A)] $v^\theta_{t,t}(x) = v^\eta_{t}(x)$ for all $t \in [0,1]$ and $x \in \Real^d$.
        \item[(B)] There exists $\tau^* > 0$ such that $\phi^\theta_{s,t}(x) = \phi^\theta_{t-\tau,t}(\phi^\theta_{s,t-\tau}(x))$ for all $x \in \Real^d$, $0 \leq s < t \leq 1$ and $\tau \leq \min\{\tau^*, t-s\}$.
    \end{enumerate}
    Then, $\phi^\theta$ satisfies the Lagrangian PDE:
    \begin{align*}
        \frac{\partial}{\partial t} \phi^\theta_{s,t}(x) = v^{\eta}_t(\phi^\theta_{s,t}(x))
    \end{align*}
\end{lemma}

\begin{proof} 
We have that
\begin{align*}
    \frac{\partial}{\partial t} \phi^\theta_{s,t}(x) = \lim_{\tau \ra 0} \frac{\phi^\theta_{s,t}(x) - \phi^{\theta}_{s,t-\tau}(x) }{\tau}.
\end{align*}
By Property (B), we know that $\phi^\theta_{s,t}(x) = \phi^\theta_{t-\tau,t}(\phi^\theta_{s,t-\tau}(x))$ when $\tau$ is small enough. Hence,
\begin{align*}
    \frac{\partial}{\partial t} \phi^\theta_{s,t}(x) 
    &= \lim_{\tau \ra 0} \frac{\phi^\theta_{t-\tau,t}(\phi^{\theta}_{s,t-\tau}(x)) - \phi^{\theta}_{s,t-\tau}(x) }{\tau} \\
    &= \lim_{\tau \ra 0} \frac{\phi^{\theta}_{s,t-\tau}(x) + \tau v_{t-\tau,t}^\theta(\phi^{\theta}_{s,t-\tau}(x)) - \phi^{\theta}_{s,t-\tau}(x) }{\tau} \\
    &= \lim_{\tau \ra 0} v_{t-\tau,t}^\theta(\phi^{\theta}_{s,t-\tau}(x)) \\
    &= v_{t,t}^\theta(\phi^{\theta}_{s,t}(x)) \\
    &= v_{t}^\eta(\phi^{\theta}_{s,t}(x)).
\end{align*}
The last line follows from Property (A).
\end{proof}

The next theorem characterizes the output of a model successfully trained with the ITVM loss. We assume that $\phi^\eta_{s,s-\tau}(x_s)$ is computed with a one-step ODE solver such as Euler's, Heun's, or Runge-Kutta. Let $\mcal{S}^{\eta,h}_{s,t}(x)$ denotes the result of using the solver to compute $\phi_{s,t}^\eta(x)$ with step size $h$. 

\begin{theorem}
Assume the following statements are true.
\begin{enumerate}
    \item[(a)] $\tau = 1/N$ where $N$ is positive integer. 
    \item[(b)] In the IAVM loss, we compute $\phi^\eta_{s,t}(x_s)$ by running the ODE solver for one step with step size $\tau$. In other words, the IAVM loss is given by
    \begin{align*}
        \mcal{L}_{\operatorname{IAVM}}(x) = E_{\substack{s \sim \mcal{U}[0,1-\tau],\\x_s \sim p_s}} \bigg[ \bigg\| v^\theta_{s,s+\tau}(x_s) - \frac{\mcal{S}_{s,s+\tau}^{\eta,\tau}(x_s) - x_s}{\tau} \bigg\|^2 \bigg]
    \end{align*}
    \item[(b)] $v^\theta$, $v^\eta$, $\phi^\theta$ and $\mcal{S}^{\eta,\tau}$, as functions, are continuous.         
    \item[(c)] The TTFM $\phi^{\theta}$ has been trained to the point that $\langle \theta \rangle = \theta$. In other words, the parameters have reached a steady state.
    \item[(d)] $\mcal{L}_{\operatorname{ITVM}} = 0$.
\end{enumerate}
Then,
\begin{align*}
    \phi_{m\tau, n\tau}^{\theta}(x) = \mcal{S}_{m\tau, n\tau}^{\eta,\tau}(x).
\end{align*}
for all $x \in \Real^d$ and integers $m, n$ such that $0 \leq m \leq n \leq N$.
\end{theorem}

\begin{proof}
We have that
\begin{align*}
    \mcal{L}_{\operatorname{ITVM}} = \mcal{L}_{\operatorname{IIVM}} + \mcal{L}_{\operatorname{IAVM}} + \mcal{L}_{\operatorname{TVM}} = 0.
\end{align*}
Because all terms are expectations of non-negative random variables, it must be the case that the individual terms must also be zero.

Consider the TVM term.
\begin{align*}
    0 
    = \mcal{L}_{\operatorname{TVM}} 
    &= E_{\substack{s \sim \mcal{U}[0,1-\tau],\\ t \sim \mcal{U}[s+\tau,1] ,\\x_s \sim p_s}}\bigg[\bigg\| \frac{\phi_{s,t}^\theta(x_s) - \phi_{s,t-\tau}^\theta(x_s)}{\tau} - v^{\langle \theta \rangle}_{t-\tau,t}(\phi_{s,t-\tau}^{[\theta]}(x_s))  \bigg\|^2\bigg] \\
    &= E_{\substack{s \sim \mcal{U}[0,1-\tau],\\ t \sim \mcal{U}[s+\tau,1] ,\\x_s \sim p_s}}\bigg[\bigg\| \frac{\phi_{s,t}^\theta(x_s) - \phi_{s,t-\tau}^\theta(x_s)}{\tau} - v^{ \theta }_{t-\tau,t}(\phi_{s,t-\tau}^{\theta }(x_s))  \bigg\|^2\bigg].
\end{align*}
As $p_s = p_{\data} * \mcal{N}(0, \sigma_s^2 I)$ where $\sigma_s = 1 - (1 - \sigma_{\min})s > 0$, it follows that $p_s$ has infinite support like $\mcal{N}(0, \sigma_s^2 I)$. Moreover, because $\phi^\theta$ and $v^\theta$ are continuous, it follows that the expression inside the expectation bracket is also continuous. It follows that
\begin{align*}
    \bigg\| \frac{\phi_{s,t}^\theta(x_s) - \phi_{s,t-\tau}^\theta(x_s)}{\tau} - v^{\theta}_{t-\tau,t}(\phi_{s,t-\tau}^\theta(x_s)) \bigg\|^2
\end{align*}
must be zero for any $x_s \in \Real^d$ and $0 \leq s < s+\tau \leq t \leq 1$. Otherwise, the expectation would not be zero. As a result, the consistency property
\begin{align}
    \frac{\phi_{s,t}^\theta(x_s) - \phi_{s,t-\tau}^\theta(x_s)}{\tau} &= v^{\theta}_{t-\tau,t}(\phi_{s,t-\tau}^\theta(x_s)) \notag \\    
    \phi_{s,t}^\theta(x_s) &= \phi_{s,t-\tau}^\theta(x_s) + \tau v^{\theta}_{t-\tau,t}(\phi_{s,t-\tau}^\theta(x_s)) \notag \\
    \phi_{s,t}^\theta(x_s) &= \phi_{t-\tau,t}^\theta(\phi_{s,t-\tau}^\theta(x_s)) \label{eqn:tvm-consistency}
\end{align}
holds for all $x_s$, $s$, and $t$ in the same ranges as above.

Next, consider the IAVM term. We have that it must be zero. So,
\begin{align*}
    \mcal{L}_{\operatorname{IAVM}} 
    = E_{\substack{s \sim \mcal{U}[0,1-\tau],\\x_s \sim p_s}} \bigg[ \bigg\| v^\theta_{s,s+\tau}(x_s) - \frac{\mcal{S}^{\eta,\tau}_{s,s+\tau}(x_s)-x_s}{\tau} \bigg\|^2 \bigg] = 0.
\end{align*}
Using an argument similar to what we applied to the TVM loss, we can deduce that
\begin{align}
v^\theta_{s,s+\tau}(x_s) &= \frac{\mcal{S}^{\eta,\tau}_{s,s+\tau}(x_s)-x_s}{\tau} \notag \\
x_s + \tau v^\theta_{s,s+\tau}(x_s) &= \mcal{S}^{\eta,\tau}_{s,s+\tau}(x_s) \notag \\
\phi^\theta_{s,s+\tau}(x_s) &= \mcal{S}^{\eta,\tau}_{s,s+\tau}(x_s) \label{eqn:iavm-agreement}
\end{align}
for all $s \in [0,1-\tau]$ and $x_s \in \Real^d$.

Let $m$ and $n$ be integers such that $0 \leq m \leq n \leq N$. With \eqref{eqn:tvm-consistency}, we can show that
\begin{align*}
    \phi^{\theta}_{m\tau, n\tau}(x) 
    &= \phi^{\theta}_{(n-1)\tau, n\tau}(\phi^{\theta}_{m\tau, (n-1)\tau}(x)) \\
    &= \phi^{\theta}_{(n-1)\tau, n\tau}( \phi^{\theta}_{(n-2)\tau, (n-1)\tau}(\phi^{\theta}_{m\tau, (n-2)\tau}(x)) ) \\
    &\ \ \vdots \\
    &= \phi^{\theta}_{(n-1)\tau, n\tau}( \phi^{\theta}_{(n-2)\tau, (n-1)\tau}(\dotsb (\phi^{\theta}_{(m+1)\tau, (m+2)\tau} (\phi^{\theta}_{m\tau, (m+1)\tau}(x))) \dotsb) )
\end{align*}
With \eqref{eqn:iavm-agreement}, we have that 
\begin{align*}
    \phi^{\theta}_{m\tau, n\tau}(x) 
    &= \mcal{S}^{\eta,\tau}_{(n-1)\tau, n\tau}( \mcal{S}^{\eta,\tau}_{(n-2)\tau, (n-1)\tau}(\dotsb (\mcal{S}^{\eta,\tau}_{(m+1)\tau, (m+2)\tau} ( \mcal{S}^{\eta,\tau}_{m\tau, (m+1)\tau}(x))) \dotsb) ) \\
    &= \mcal{S}^{\eta,\tau}_{m\tau, n\tau}(x).
\end{align*}
The last line follows from the fact that $\mcal{S}$ is a one-step ODE solver. We are done.
\end{proof}

\section{Metrics Used for Performance Comparison} \label{sec:metrics}

\subsection{KL Divergence}

We use the Kullback--Leibler (KL) divergence to measure performance of trained student TTFMs for 2D and tabular datasets. In this section, we explain how the metric is computed.

\subsubsection{Preliminary}

Let $\varphi: \Real^d \ra \Real^d$ be a diffeomorphism; i.e., a differentiable bijection. Given a probability distribution $p$ on $\Real^d$, the \emph{push-forward} of $p$ by $\varphi$ is new probability distribution $q$ that arises from the following process.
\begin{enumerate}
    \item Sample $x \sim p$.
    \item Compute $\varphi(x)$ and return the result.
\end{enumerate}
We denote the push-forward with $[\varphi]_*\, p$. It is well-known that,
\begin{align}
    q(y) = \big( [\varphi]_*\, p \big)(x) = \frac{p(\varphi^{-1}(y))}{| \det \nabla \varphi(\varphi^{-1}(y)) |} \label{eqn:push-forward-formula}
\end{align}
for any $y \in \Real^d$. Here, $\nabla \varphi(x)$ denotes the Jacobian matrix of $\varphi$ at point $x$.

\subsubsection{Problem setup} \label{sec:kl-problem-setup}

Recall from Section~\ref{sec:preliminary}, we want to model a probability path $\{ p_t: t \in [0,1] \}$. To do so, We first train a flow matching model $v^\eta$ such that its implicit flow $\phi^\eta$ satisfies
\begin{align*}
    p_t \approx [\phi^\eta_{0,t}]_*\ p_0 = [\phi^\eta_{0,t}]_*\ \mcal{N}(0,I).
\end{align*}
Notice that, while the implicit flow $\phi^\eta$ is a function of signature $\Real \times \Real \times \Real^d \ra \Real^d$, we view $\phi^\eta_{0,t}(\cdot)$ (i.e., $\phi^\eta$ after the initial and terminal time have been fixed) as a diffeomorphism of signature $\Real^d \ra \Real^d$. Let $p^\eta_t$ denote the approximation of $p_t$ created with a push forward by the implicit flow $\phi^{\eta}_{0,t}$. In other words,
\begin{align*}
    p^\eta_t = [\phi^\eta_{0,t}]_*\, p_0.
\end{align*}
In Section~\ref{sec:experiments}, we want to create samples from $p_1$. So, we only care about $p^\eta_1$.

To speed up sampling, we train a student TTFM $\phi^\theta$ using $v^\eta$ as its teacher. Using the student model, we can generate data samples by evaluating the student models $K \in \mathbb{N}$ times as follows.
\begin{itemize}
    \item Sample $x \sim p_0$.
    \item Compute $y = \phi^{\theta}_{\frac{K-1}{K},1}(\phi^{\theta}_{\frac{K-2}{K},\frac{K-1}{K}}(\dotsb (\phi^{\theta}_{\frac{2}{K}, \frac{3}{K}}(\phi^{\theta}_{\frac{1}{K},\frac{2}{K}}(\phi^{\theta}_{0,\frac{1}{K}}(x)))) \dotsb )) $
    \item Return $y$.
\end{itemize}
Using the push-forward notation, the probability distribution of $y$ generated this way is given by
\begin{align*}
    \Big[\phi^{\theta}_{\frac{K-1}{K},1}\Big]_*\,  \Big[\phi^{\theta}_{\frac{K-2}{K},\frac{K-1}{K}}\Big]_*\,  \dotsb \, \Big[\phi^{\theta}_{\frac{2}{K},\frac{3}{K}}\Big]_*\, \Big[\phi^{\theta}_{\frac{1}{K},\frac{2}{K}}\Big]_*\, \Big[\phi^{\theta}_{0,\frac{1}{K}}\Big]_*\, p_0.
\end{align*}
For brevity, let us denote it as $p^\theta_{(K)}$. In Section~\ref{sec:experiments}, we created samples using 1, 2, 4, and 8 function evaluations. So, we must work with the four distributions: $p^\theta_{(1)}$, $p^\theta_{(2)}$, $p^\theta_{(4)}$, and $p^\theta_{(8)}$.

\subsubsection{KL Divergence as Performance Indicator}

We may assess how well the student TTFM performs by measuring how much the probability distributions $p^\theta_{(1)}$, $p^\theta_{(2)}$, $p^\theta_{(4)}$, and $p^\theta_{(8)}$ differ from the teacher's distribution $p^\eta_1$. The closer the student distributions are to the teacher distribution, the better the student performs.

A well-known way to measure the difference between two probability distributions is the Kullback--Leibler (KL) divergence. Given two probability distributions $p$ and $q$ over the same domain, the KL divergence of $p$ with respect to $q$ is defined as
\begin{align*}
    \operatorname{KL}(p\|q) = E_{x \sim p} \bigg[ \log \frac{p(x)}{q(x)} \bigg] = E_{x \sim p} [ \log p(x) - \log q(x) ].
\end{align*}
Here, $q$ acts as a reference distribution that we want to measure how different $p$ is from it. As a result, we can measure the student's performance by computing $\operatorname{KL}\big(p^{\theta}_{(1)}\|p^\eta_1\big)$, $\operatorname{KL}\big(p^{\theta}_{(2)}\|p^\eta_1\big)$, $\operatorname{KL}\big(p^{\theta}_{(4)}\|p^\eta_1\big)$, and $\operatorname{KL}\big(p^{\theta}_{(8)}\|p^\eta_1\big)$.

Let us think about how to compute $\operatorname{KL}\big(p^{\theta}_{(8)}\|p^\eta_1\big)$. We have that
\begin{align*}
    \operatorname{KL}(p^\theta_{(K)}\|p^\eta_1) = E_{y \sim p^\theta_{(K)}} \big[ \log p^\theta_{(K)}(y) - \log p^\eta_1(y) \big].
\end{align*}
In order to be able to compute what is on the RHS, we must be able to do 3 things:
\begin{enumerate}
    \item[(a)] sampling $y$ from $p^\theta_{(K)}$,
    \item[(b)] evaluating $\log p^\theta_{(K)}(y)$, and
    \item[(c)] evaluating $\log p^\eta_1(y)$.
\end{enumerate}
Performing (a) is straightforward, and the process is described in Section~\ref{sec:kl-problem-setup}. We shall discuss (b) and (c) in the next sections.

\subsubsection{Evaluating $\log p^\theta_{(K)}(y)$}

Suppose that $y$ is a sample from $p^\theta_{(K)}(y)$ using the process we described in the last section. Then,
\begin{align}
y &= \phi^{\theta}_{\frac{K-1}{K},1}(\phi^{\theta}_{\frac{K-2}{K},\frac{K-1}{K}}(\dotsb (\phi^{\theta}_{\frac{2}{K}, \frac{3}{K}}(\phi^{\theta}_{\frac{1}{K},\frac{2}{K}}(\phi^{\theta}_{0,\frac{1}{K}}(x)))) \dotsb )) \label{eqn:ttfm-sampling}
\end{align}
where $x \sim p_0 = \mcal{N}(0,I)$. For $k = 1,2,\dotsc,K$, define
\begin{align*}
    y^{(k)} &= \phi^{\theta}_{\frac{k-1}{K},\frac{k}{K}}(\phi^{\theta}_{\frac{k-2}{K},\frac{k-1}{K}}(\dotsb (\phi^{\theta}_{\frac{2}{K}, \frac{3}{K}}(\phi^{\theta}_{\frac{1}{K},\frac{2}{K}}(\phi^{\theta}_{0,\frac{1}{K}}(x)))) \dotsb )),
\end{align*}
and let $y^{(0)} = x$ so that $y^{(K)}=y$ and $y^{(k)}=\phi^{\theta}_{\frac{k-1}{K}, \frac{k}{K}}(y^{(k-1)})$. Then, applying \eqref{eqn:push-forward-formula} repeatedly, we have that
\begin{align*}
    p^\theta_{(K)}(y) = p_0(x) \prod_{k=1}^K \frac{1}{\Big| \det \nabla \phi_{\frac{k-1}{K}, \frac{k}{K}}^\theta( y^{(k-1)}) \Big|}.
\end{align*}
Hence,
\begin{align*}
    \log p^\theta_{(K)}(y) = \log p_0(x) - \sum_{k=1}^{K} \log \Big| \det \nabla \phi_{\frac{k-1}{K}, \frac{k}{K}}^\theta( y^{(k-1)}) \Big|.
\end{align*}
The above expression is easy to compute. First, $p_0$ is the standard Gaussian function, and its logarithm is simply $-\|x\|^2/2$ plus a constant. Second, we can compute each $y^{(k)}$ by saving the output of the TTFM each time we apply it in \eqref{eqn:ttfm-sampling}. Third, we can compute the Jacobian matrix and its determinant because we have direct access to the TTFM. In our implementation, we use PyTorch's \texttt{torch.autograd.functional.jacobian} to do so.

\subsubsection{Evaluating $\log p^\eta_1(y)$}

In this section, we treat $y$ as an arbitrary point in $\Real^d$ given to us by the user. We have that $p^\eta_1(y)$ is the probability of generating $y$ using the flow matching model $v^\eta$. The generation process involves sampling $x \sim p_0$ and then computing $\phi_{0,1}^\eta(x)$ by solving the initial value problem
\begin{align*}
    \frac{\partial}{\partial t} \phi^\eta_{0,t}(x) = v^\eta_t(\phi_{0,t}^\eta(x))
\end{align*}
with initial condition $\phi^\eta_{0,0}(x) = x$ using an ODE solver. In other words, we compute the integral
\begin{align*}
    \phi_{0,1}^{\eta}(x) =  x + \int_{0}^1 v^\eta_u(\phi_{0,u}^\eta(x))\, \dee u.
\end{align*}
Given $y$, we can find an $x$ such that $\phi_{0,1}^\eta(x) \approx y$ by starting from $y$ at time $t = 1$ and running the ODE solver backward in time until we reach $t = 0$.

With such an $x$, we must compute $\log p^\eta_1(y) \approx \log p_1(\phi^\eta_{0,1}(x))$. Here, we may think of the flow matching model $v^\eta$ as a \emph{continuous normalizing flow model} \cite{Chen:NODE:2018} and use the following \emph{instantaneous range of change formula}.

\begin{theorem}[instantaneous change of variable formula]
Let $v^\eta_t(\cdot)$ be a vector field that generates the marginal probability path $p_t(\cdot)$. Then,
\begin{align*}
    \frac{\partial}{\partial t} \Big( \log p_t\big(\phi_{0,t}^\eta(x)\big) \Big) = -\tr\Big( \nabla v^\eta_t \big(\phi_{0,t}^\eta(x)\big) \Big).
\end{align*}
Equivalently, 
\begin{align}
    \log p_t(\phi^\eta_{0,t}(x)) = \log p_0(x) - \int_{0}^t \tr\Big( \nabla v^\eta_u\big(\phi_{0,u}^\eta(x)\big) \Big)\, \dee u. \label{eqn:log-prob-integral}
\end{align}
\end{theorem}

In other words, we may compute $\log p_1(\phi^\eta_{0,1}(x))$ by computing the integral \eqref{eqn:log-prob-integral} with an ODE solver. We can run this ODE solver in parallel with another ODE solver that computes $\phi_{0,u}^\eta(x)$ and use the value to compute the Jacobian matrix $\nabla v^\eta_u(\phi_{0,u}^\eta(x))$, which we can then compute its trace and use it to step the ODE solver that computes the logarithm of the probability.

\subsubsection{Computing the KL divergence}

Having established how to compute both $\log p^\theta_{(K)}(y)$ and $\log p^\eta_1(y)$ given $y \sim p_{(K)}^\theta$, we now proceed to computing the KL divergence
\begin{align*}
\operatorname{KL}(p^\theta_{(K)}\|p^\eta_1) = E_{y \sim p^\theta_{(K)}} \big[ \log p^\theta_{(K)}(y) - \log p^\eta_1(y) \big].
\end{align*}
We do so with Monte Carlo integration: we sample $N$ samples $y_1$, $y_2$, $\dotsc$, $y_N$ from $p_{(K)}^\theta$ and then estimate the KL divergence with
\begin{align*}
\operatorname{KL}(p^\theta_{(K)}\|p^\eta_1) \approx \frac{1}{N} \sum_{i=1}^N \Big( \log p^\theta_{(K)}(y_i) - \log p^\eta_1(y_i) \Big).
\end{align*}
Theoretically, KL divergence is non-negative, but the estimate above might be negative. To ensure that the estimate is always non-negative, we use a method by Schulman~\yrcite{Schulman:2020}. Let
\begin{align*}
    r(y) = \frac{p^\eta_1(y)}{p^\theta_{(K)}(y)}.
\end{align*}
It follows that $r(y) -1 - \log r(y) \geq 0$ because $e^{x} \geq 1 + x$ for all $x \in \Real$. Now,
\begin{align*}
    E_{y \sim p^\theta_{(K)}} [r(y) - 1 - \log r(y)] 
    &= E_{y \sim p^\theta_{(K)}}\bigg[ \frac{p^\eta_1(y)}{p^\theta_{(K)}(y) }  \bigg] - E_{y \sim p^\theta_{(K)}}[1] - E_{y \sim p^\theta_{(K)}} \bigg[ \log \frac{p^\eta_1(y)}{p^\theta_{(K)}(y)} \bigg] \\
    &= \int p^\theta_{(K)}(y) \frac{p^\eta_1(y)}{p^\theta_{(K)}(y)}\, \dee x - 1 + E_{y \sim p^\theta_{(K)}} \bigg[ \log \frac{p^\theta_{(K)}(y)}{p^\eta_1(y)} \bigg] \\
    &= \int p^\eta_1(y)\, \dee x - 1 + \operatorname{KL}(p^\theta_{(K)}\|p^\eta_1) \\
    &= \operatorname{KL}(p^\theta_{(K)}\|p^\eta_1).
\end{align*}
As a result,
\begin{align*}
\frac{1}{N} \sum_{i=1}^N \big( r(y_i) - 1 - \log r(y_i) \big)
\end{align*}
is an unbiased estimate of $\operatorname{KL}(p^\theta_{(K)}\|p^\eta_1)$ that is always non-negative. We use the above expression as an estimate of the KL divergence.

\subsubsection{Implementation Details}

When evaluating TTFMs for 2D and tabular datasets, we sample $N = 50,000$ samples from $p^\theta_{(K)}$. For each $y$, we compute $\log p^\eta_1(y)$ with Heun's method using 100 integration steps.

\subsection{Fr\'{e}chet Inception Distance}

For image datasets, it is not practical to compute the KL divergence using the method in the last section because the Jacobian matrix would be very large, making computation extremely time-consuming. We instead use the Fr\'{e}chet Inception Distance (FID), which is a standard metric for assessing image generation models. 

It is well known that FID varies with the number of samples used to compute it \cite{Chong:2020}. We generated as many samples as the number of training samples in the dataset. That is, we generated 60,000 samples for MNIST and 50,000 samples for CIFAR-10.

There are many implementations of FID computation, and subtle details such as how images are resized can affect the final score \cite{Parmar:2021}. We used an implementation by Karras \etal~\yrcite{Karras:2022}, available in the Github repository \url{https://github.com/NVlabs/edm}, as it seems to be widely used by previous works.

\section{Experimental Setup and Data for ``Comparison with Other Losses'' (Section~\ref{sec:comparison-with-other-losses})} \label{sec:comparison-setup-and-data}

\subsection{2D Datasets}

\subsubsection{Model Architectures} \label{sec:2d-model-architecture}

Both the teacher and student models are MLPs with ELU \cite{Clevert:2016} activation functions. 

A teacher model accepts a time $t \in [0,1]$ and a point $x \in \Real^2$. The time is converted to a 256-dimensional vector via positional encoding \cite{Vaswani:2017} and then concatenated with $x$ to form a 258-dimensional vector. This vector is then sequentially processed by 8 hidden blocks. Each block has a fully-connected layer with 512 units followed by the ELU activation function. After the last hidden block, the 512-dimensional vector is fed to a fully-connected unit that converts it to a point in $\Real^2$, which becomes the output of the network.

\begin{figure}[h]
    \centering
    \includegraphics[width=\linewidth]{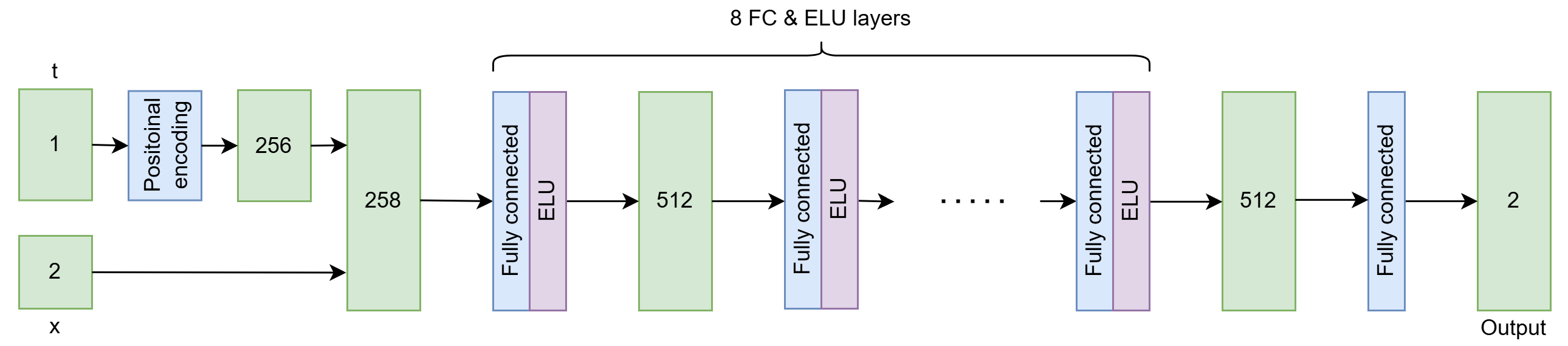}
    \caption{Architecture of the teacher flow matching model for 2D datasets.}
    \label{fig:2d-teacher-architecture}
\end{figure}

A student model's main body is an MLP like that of the teacher's, but it takes as input two times $s, t \in \Real$ in addition to a point $x \in \Real^2$. The two times are converted to two 256-dimensional positional encodings. They are then concatenated with the input point $x$ to form a vector with $256+256+2 = 514$ dimensions. This vector is then processed by 8 hidden blocks and a final fully-connected layer like what is done in the teacher model. However, each hidden blocks now has 1,024 units instead of 512. We note that the main body represents the average velocity model (AVM) $v^\theta_{s,t}(x)$. The TTFM is $\phi^{\theta}_{s,t}(x) = x + (t-s)v^\theta_{s,t}(x)$, and we convert an AVM to a TTFM with a wrapper class. 

\begin{figure}[h]
    \centering
    \includegraphics[width=\linewidth]{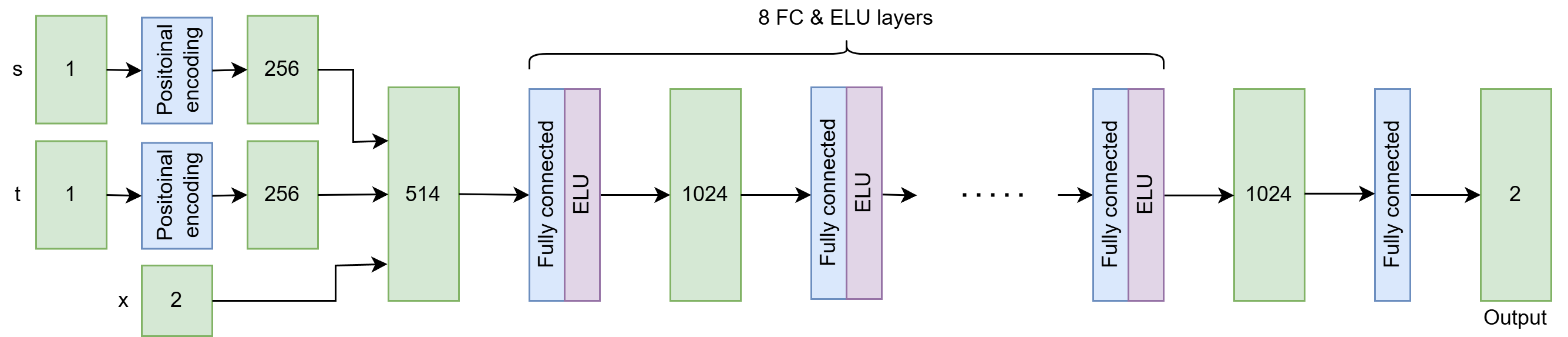}
    \caption{Architecture of the student TTFM for 2D datasets.}
    \label{fig:2d-student-architecture}
\end{figure}

\subsubsection{Training Process} \label{sec:2d-model-training-process}

Both the teacher and student models are trained with the same process. We use the Adam optimizer \cite{Kingma:2015} with $\beta_1 = 0.9$ and $\beta_2 = 0.999$. The batch size is $1,000$. Learning rate increases linearly from $0$ to $10^{-4}$ in the first 10 training iterations and remains at that value until training finishes. Training lasts until the model under training has been shown 100M training examples, which amounts to 100,000 training iterations. Independent of the EMA of the trained model parameters used in some loss functions, we maintain another EMA with a decay rate of $0.999$, and we use this EMA at test time to generate samples and compute performance statistics.

In the next 4 sessions, we show KL divergence metrics for student models trained with our ITVM loss and other baselines. We also show samples generated by the models with NFE = 1.

\newpage

\subsubsection{2D Dataset Results: 5LOBES}

\begin{center}
    \begin{tabular}{|l|r|r|r|r|r|}
        \hline
        \multirow{2}{*}{Loss} 
        & \multicolumn{4}{c|}{KL Divergence ($\downarrow$)} 
        & \multirow{2}{*}{Fused rank}\\
        \cline{2-5}
        & \multicolumn{1}{c|}{NFE = 1} & \multicolumn{1}{c|}{NFE = 2} & \multicolumn{1}{c|}{NFE = 4} & \multicolumn{1}{c|}{NFE = 8} & \\
        \hline
EFMD & {\color{Red} 0.000152 (01)}& {\color{Blue} 0.000141 (03)}& 0.000130 (06)& 0.000122 (06)& 04\\
LFMD & {\color{Green} 0.000180 (02)}& 0.000143 (04)& {\color{Blue} 0.000121 (03)}& 0.000106 (05)& {\color{Blue} 03}\\
PID & 0.000330 (07)& 0.000338 (07)& 0.000306 (07)& 0.000280 (07)& 07\\
\hline
ITVM ($\mu = 0.0$) & 0.000232 (06)& 0.000160 (06)& 0.000128 (05)& 0.000094 (04)& 06\\
ITVM ($\mu = 0.9$) & {\color{Blue} 0.000183 (03)}& {\color{Green} 0.000140 (02)}& {\color{Green} 0.000111 (02)}& {\color{Green} 0.000083 (02)}& {\color{Green} 02}\\
ITVM ($\mu = 0.99$) & 0.000187 (04)& {\color{Red} 0.000138 (01)}& {\color{Red} 0.000107 (01)}& {\color{Red} 0.000078 (01)}& {\color{Red} 01}\\
ITVM ($\mu = 0.999$) & 0.000199 (05)& 0.000150 (05)& 0.000122 (04)& {\color{Blue} 0.000085 (03)}& 05\\
\hline
    \end{tabular}
\end{center}

\begin{center}
    \small
\begin{tabular}{@{\hskip 0.01\linewidth}c@{\hskip 0.01\linewidth}c@{\hskip 0.01\linewidth}c@{\hskip 0.01\linewidth}c@{\hskip 0.01\linewidth}}
    \includegraphics[width=0.23\linewidth]{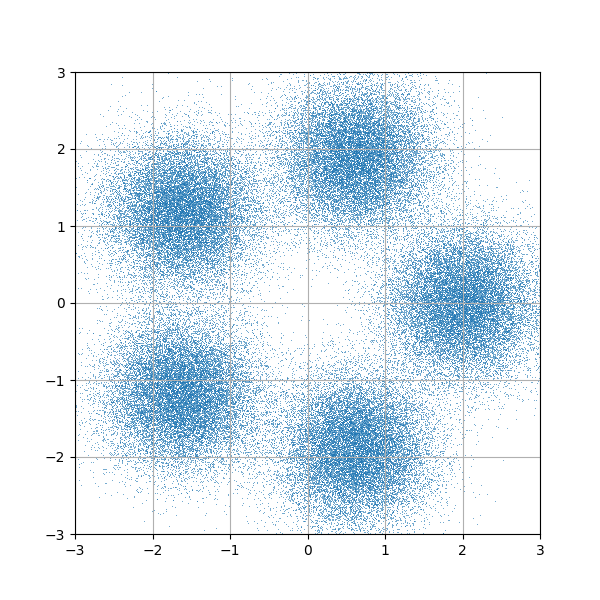} &
    \includegraphics[width=0.23\linewidth]{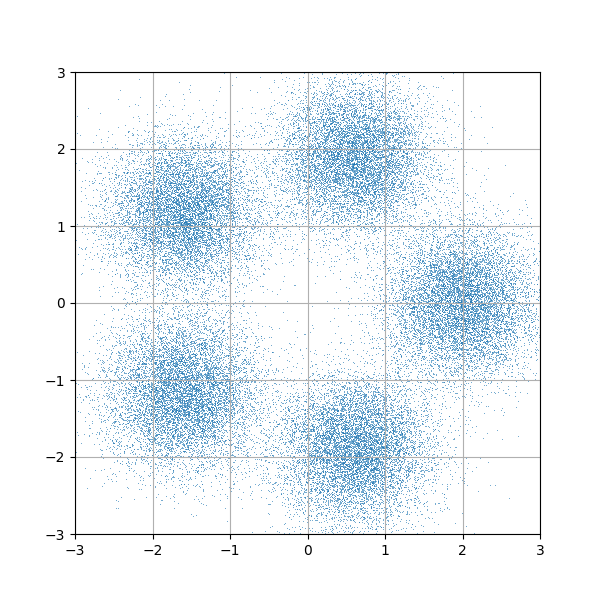} &
    &
    \\
    Dataset &
    Teacher (NFE = 100) &
    &
    \\    
    \includegraphics[width=0.23\linewidth]{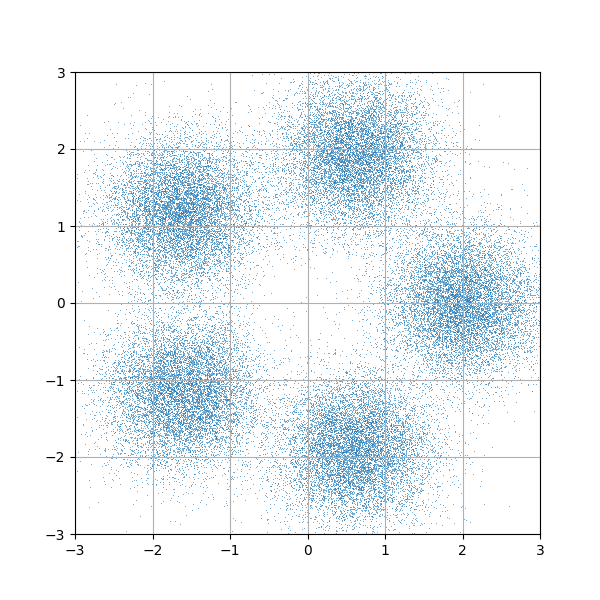} &
    \includegraphics[width=0.23\linewidth]{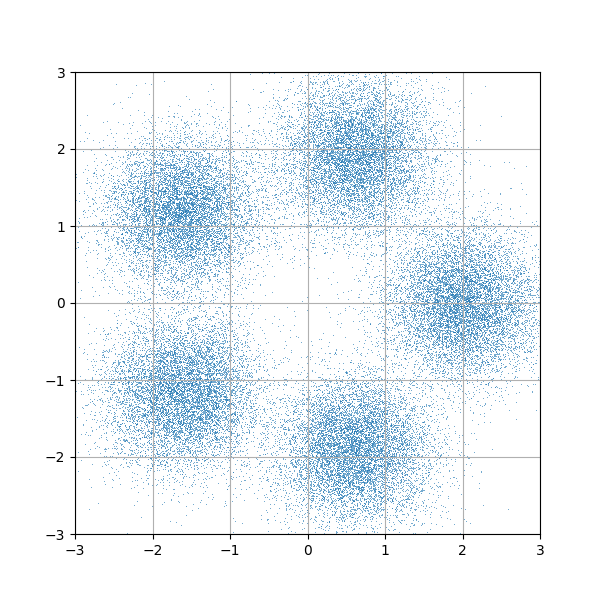} &
    \includegraphics[width=0.23\linewidth]{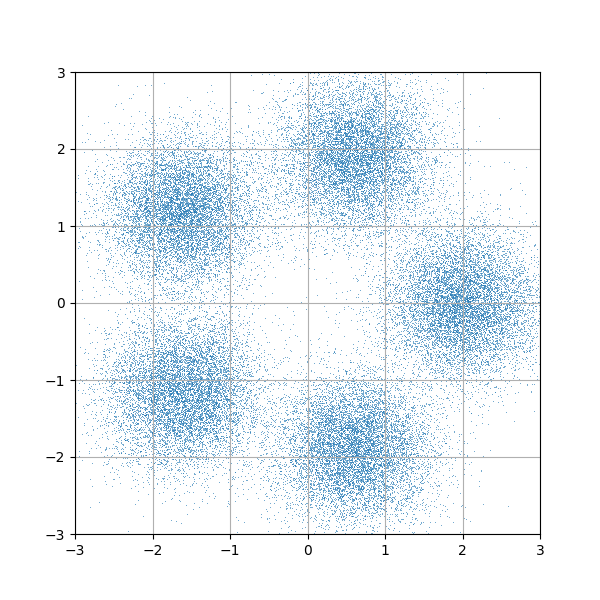} &    
    \\
    EFMD &
    LFMD &
    PID &
    \\
    \includegraphics[width=0.23\linewidth]{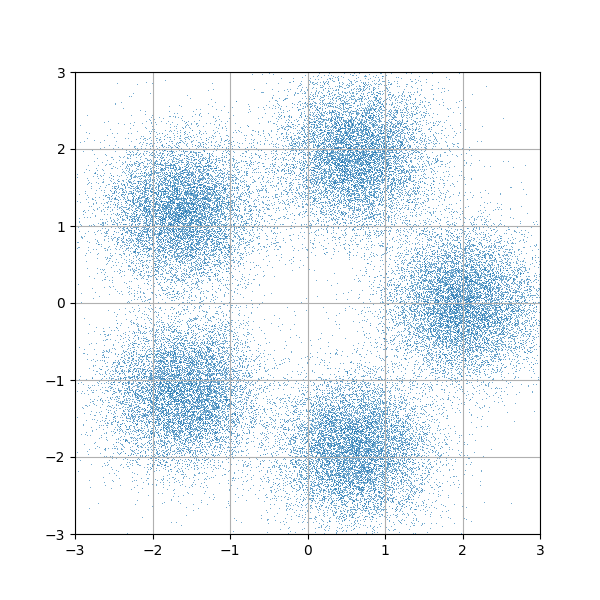} &
    \includegraphics[width=0.23\linewidth]{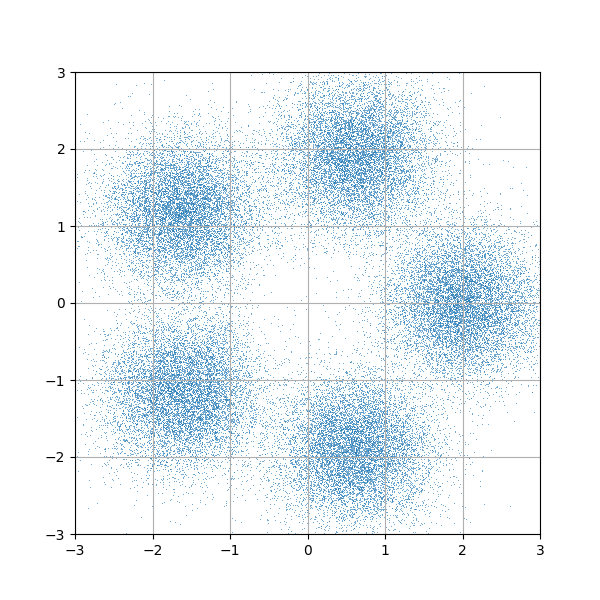} &
    \includegraphics[width=0.23\linewidth]{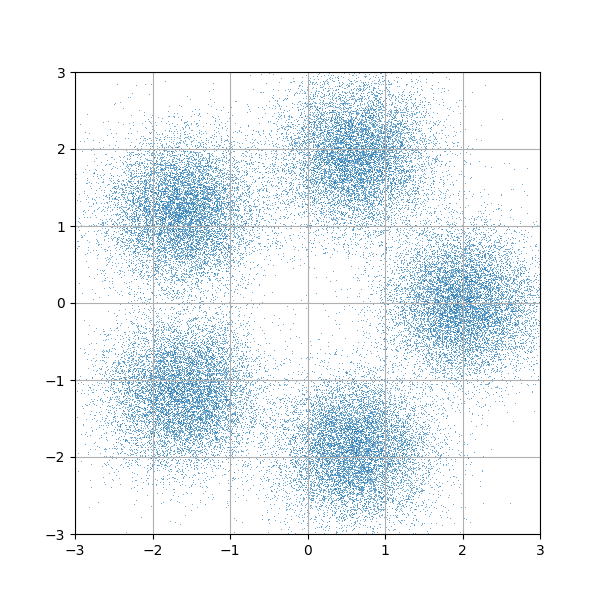} &
    \includegraphics[width=0.23\linewidth]{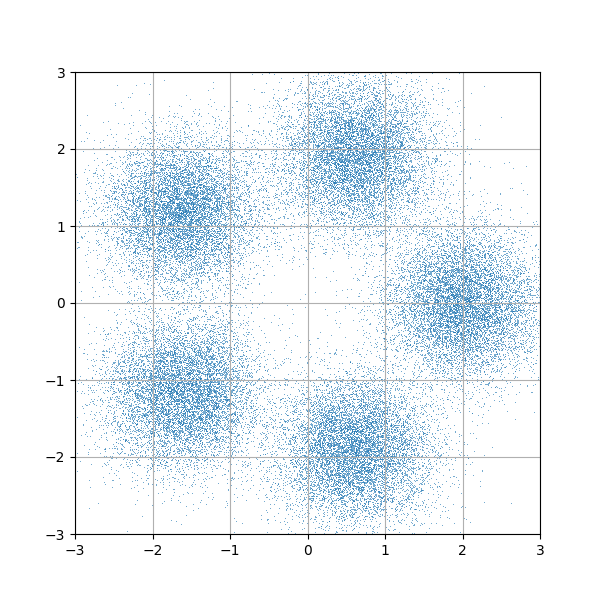} \\
    ITVM $(\mu = 0)$ &
    ITVM $(\mu = 0.9)$ &
    ITVM $(\mu = 0.99)$ &
    ITVM $(\mu = 0.999)$ \\
    & & & \\
    \multicolumn{4}{c}{Samples of student models were generated with 1 function evaluation.} 
\end{tabular}        
\end{center}

\newpage 

\subsubsection{2D Dataset Results: CHECKER}

\begin{center}
    \centering
    \begin{tabular}{|l|r|r|r|r|r|}
        \hline
        \multirow{2}{*}{Loss} 
        & \multicolumn{4}{c|}{KL Divergence ($\downarrow$)} 
        & \multirow{2}{*}{Fused rank}\\
        \cline{2-5}
        & \multicolumn{1}{c|}{NFE = 1} & \multicolumn{1}{c|}{NFE = 2} & \multicolumn{1}{c|}{NFE = 4} & \multicolumn{1}{c|}{NFE = 8} & \\
        \hline
EFMD & 0.069939 (07)& 0.056285 (07)& 0.033442 (07)& 0.019348 (06)& 07\\
LFMD & 0.027366 (05)& 0.015596 (05)& 0.012062 (05)& 0.010455 (05)& 05\\
PID & 0.025123 (04)& 0.014586 (04)& 0.010842 (04)& {\color{Blue} 0.008984 (03)}& 04\\
\hline
ITVM ($\mu = 0.0$) & 0.030792 (06)& 0.024484 (06)& 0.022426 (06)& 0.021067 (07)& 06\\
ITVM ($\mu = 0.9$) & {\color{Blue} 0.018558 (03)}& {\color{Red} 0.009030 (01)}& {\color{Red} 0.008515 (01)}& {\color{Red} 0.007128 (01)}& {\color{Red} 01}\\
ITVM ($\mu = 0.99$) & {\color{Green} 0.017441 (02)}& {\color{Blue} 0.010424 (03)}& {\color{Blue} 0.009856 (03)}& 0.009202 (04)& {\color{Blue} 03}\\
ITVM ($\mu = 0.999$) & {\color{Red} 0.014908 (01)}& {\color{Green} 0.009812 (02)}& {\color{Green} 0.009107 (02)}& {\color{Green} 0.008500 (02)}& {\color{Green} 02}\\
\hline
    \end{tabular}
\end{center}

\begin{center}
    \small
\begin{tabular}{@{\hskip 0.01\linewidth}c@{\hskip 0.01\linewidth}c@{\hskip 0.01\linewidth}c@{\hskip 0.01\linewidth}c@{\hskip 0.01\linewidth}}
    \includegraphics[width=0.23\linewidth]{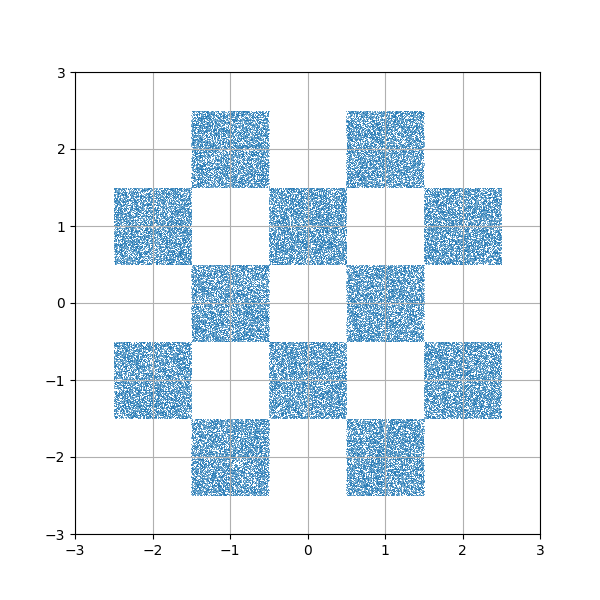} &
    \includegraphics[width=0.23\linewidth]{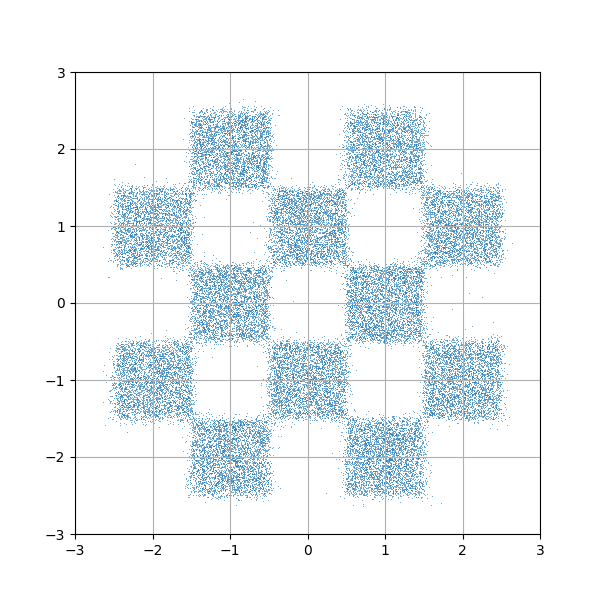} &
    &
    \\
    Dataset &
    Teacher (NFE = 100) &
    &
    \\    
    \includegraphics[width=0.23\linewidth]{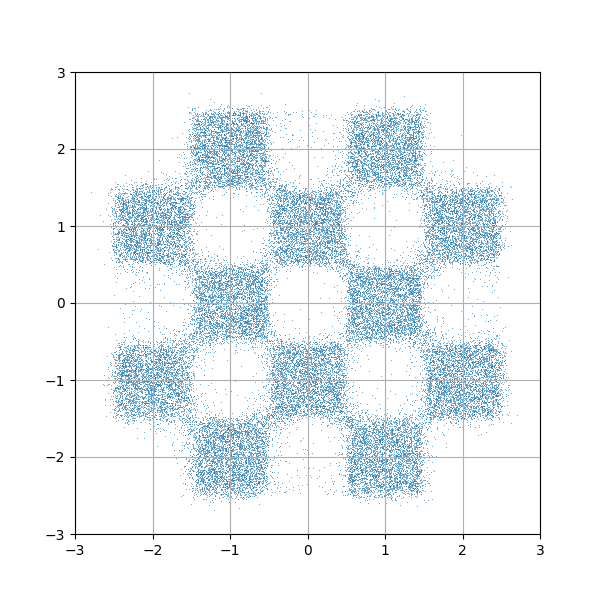} &
    \includegraphics[width=0.23\linewidth]{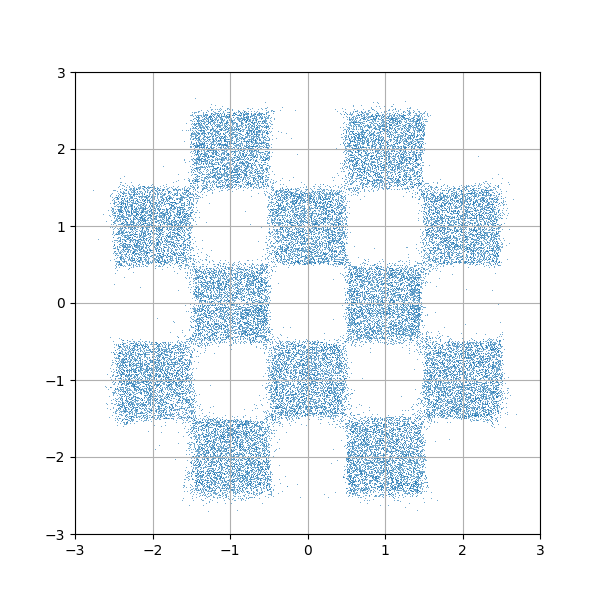} &
    \includegraphics[width=0.23\linewidth]{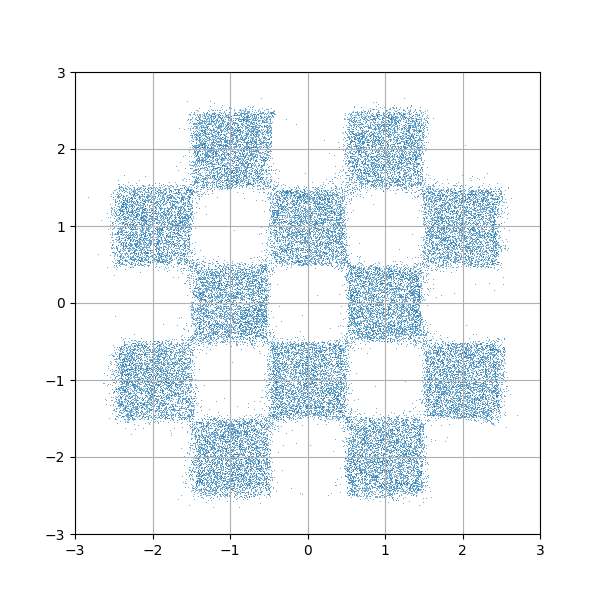} &    
    \\
    EFMD &
    LFMD &
    PID &
    \\
    \includegraphics[width=0.23\linewidth]{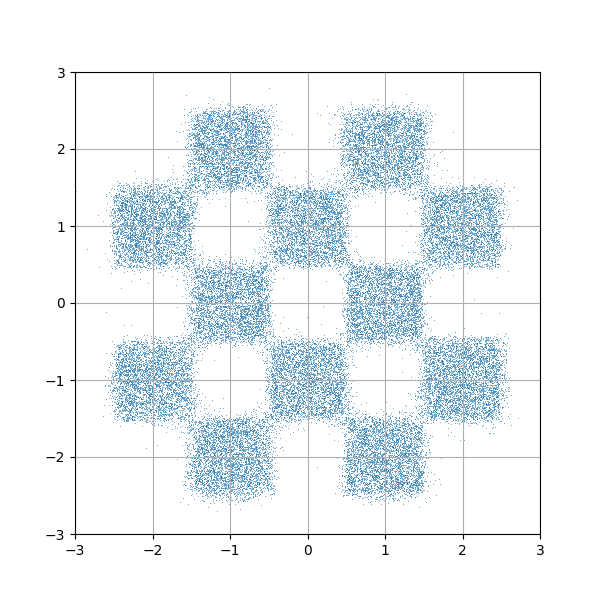} &
    \includegraphics[width=0.23\linewidth]{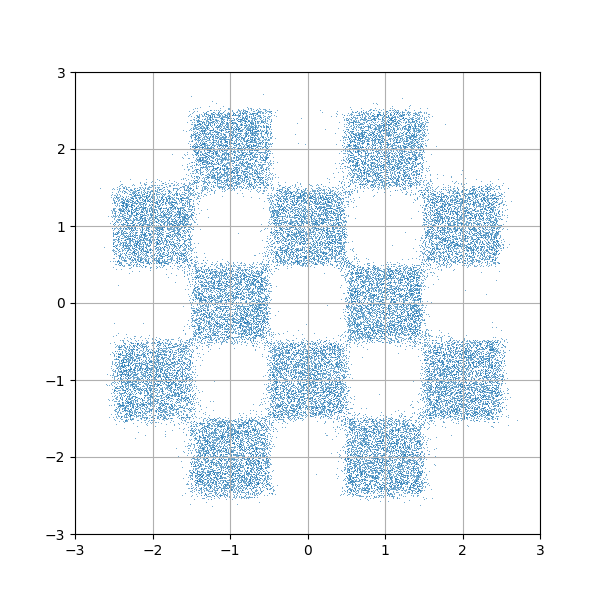} &
    \includegraphics[width=0.23\linewidth]{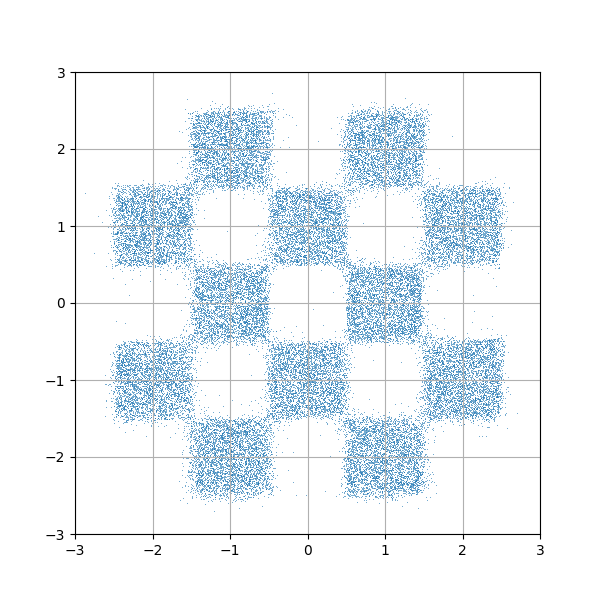} &
    \includegraphics[width=0.23\linewidth]{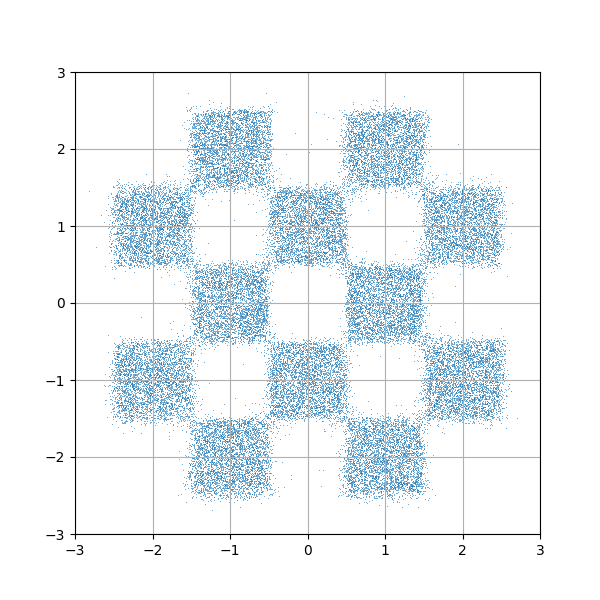} \\
    ITVM $(\mu = 0)$ &
    ITVM $(\mu = 0.9)$ &
    ITVM $(\mu = 0.99)$ &
    ITVM $(\mu = 0.999)$ \\
    & & & \\
    \multicolumn{4}{c}{Samples of student models were generated with 1 function evaluation.}
    \end{tabular}        
\end{center}

\newpage

\subsubsection{2D Dataset Results: WORD}

\begin{center}
    \centering
    \begin{tabular}{|l|r|r|r|r|r|}
        \hline
        \multirow{2}{*}{Loss} 
        & \multicolumn{4}{c|}{KL Divergence ($\downarrow$)} 
        & \multirow{2}{*}{Fused rank}\\
        \cline{2-5}
        & \multicolumn{1}{c|}{NFE = 1} & \multicolumn{1}{c|}{NFE = 2} & \multicolumn{1}{c|}{NFE = 4} & \multicolumn{1}{c|}{NFE = 8} & \\
        \hline
EFMD & 0.479471 (05)& 0.355560 (06)& 0.231161 (06)& 0.151074 (06)& 06\\
LFMD & 0.734259 (07)& 0.167223 (05)& 0.128198 (05)& {\color{Blue} 0.092998 (03)}& 05\\
PID & 0.419129 (04)& 0.160760 (04)& {\color{Blue} 0.118621 (03)}& {\color{Red} 0.076293 (01)}& {\color{Blue} 03}\\
\hline
ITVM ($\mu = 0.0$) & 0.563841 (06)& 0.356131 (07)& 0.318504 (07)& 0.292005 (07)& 07\\
ITVM ($\mu = 0.9$) & {\color{Red} 0.247495 (01)}& {\color{Green} 0.130479 (02)}& {\color{Red} 0.096200 (01)}& {\color{Green} 0.085024 (02)}& {\color{Red} 01}\\
ITVM ($\mu = 0.99$) & {\color{Green} 0.264598 (02)}& {\color{Red} 0.129752 (01)}& {\color{Green} 0.111104 (02)}& 0.107310 (04)& {\color{Green} 02}\\
ITVM ($\mu = 0.999$) & {\color{Blue} 0.267679 (03)}& {\color{Blue} 0.138303 (03)}& 0.119014 (04)& 0.108817 (05)& 04\\
\hline
    \end{tabular}
\end{center}

\begin{center}
    \small
\begin{tabular}{@{\hskip 0.01\linewidth}c@{\hskip 0.01\linewidth}c@{\hskip 0.01\linewidth}c@{\hskip 0.01\linewidth}c@{\hskip 0.01\linewidth}}
    \includegraphics[width=0.23\linewidth]{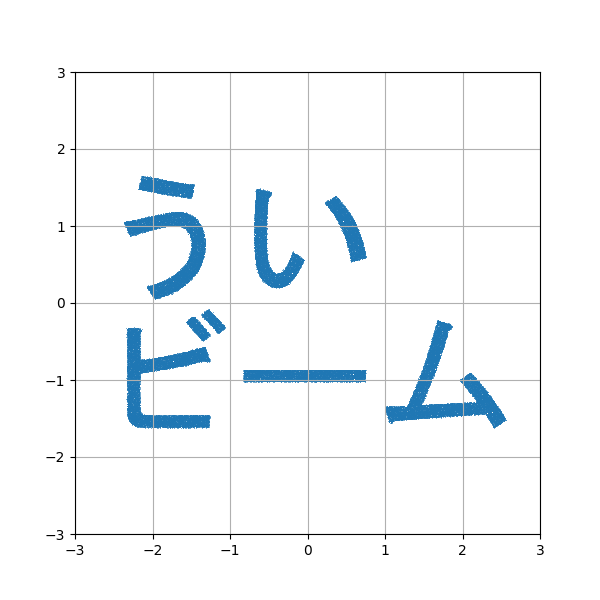} &
    \includegraphics[width=0.23\linewidth]{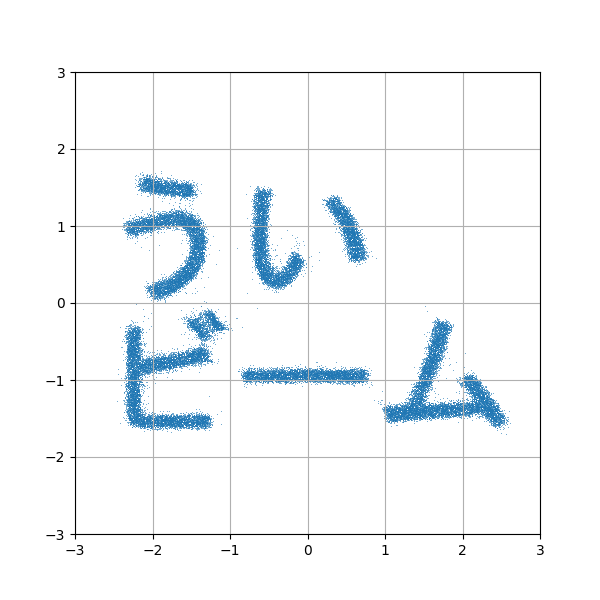} &
    &
    \\
    Dataset &
    Teacher (NFE = 100) &
    &
    \\    
    \includegraphics[width=0.23\linewidth]{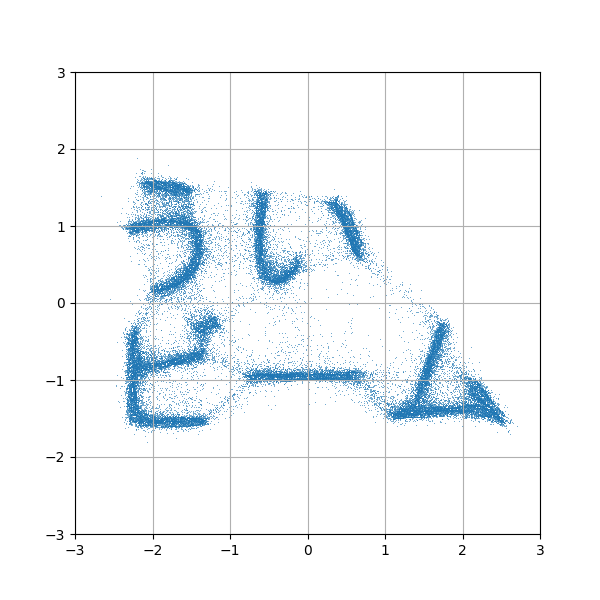} &
    \includegraphics[width=0.23\linewidth]{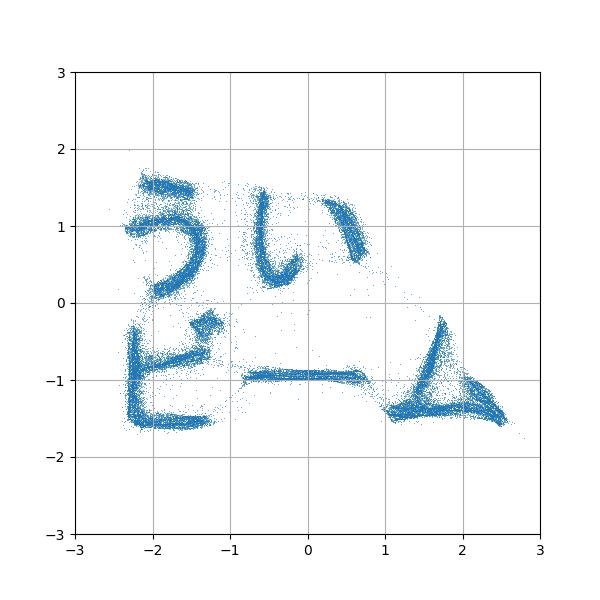} &
    \includegraphics[width=0.23\linewidth]{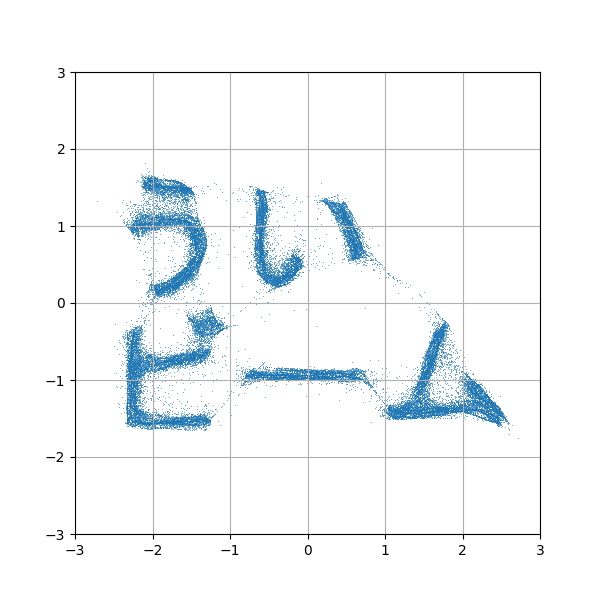} &    
    \\
    EFMD &
    LFMD &
    PID &
    \\
    \includegraphics[width=0.23\linewidth]{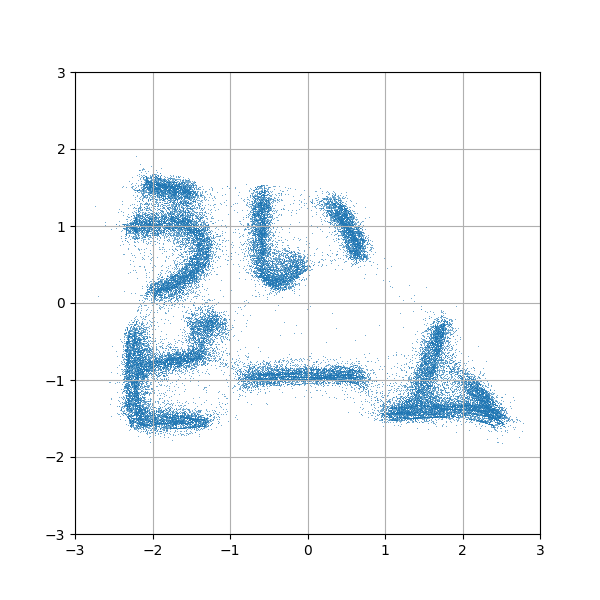} &
    \includegraphics[width=0.23\linewidth]{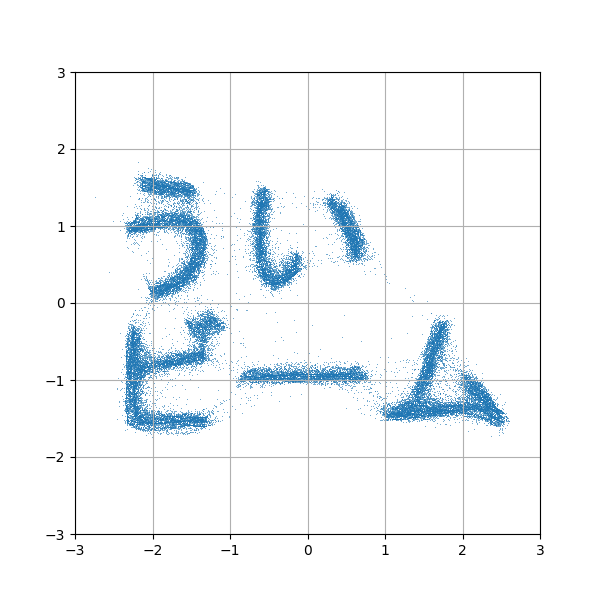} &
    \includegraphics[width=0.23\linewidth]{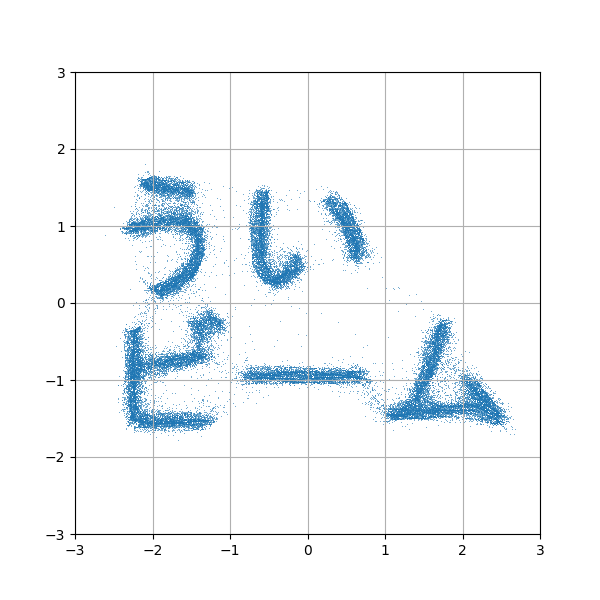} &
    \includegraphics[width=0.23\linewidth]{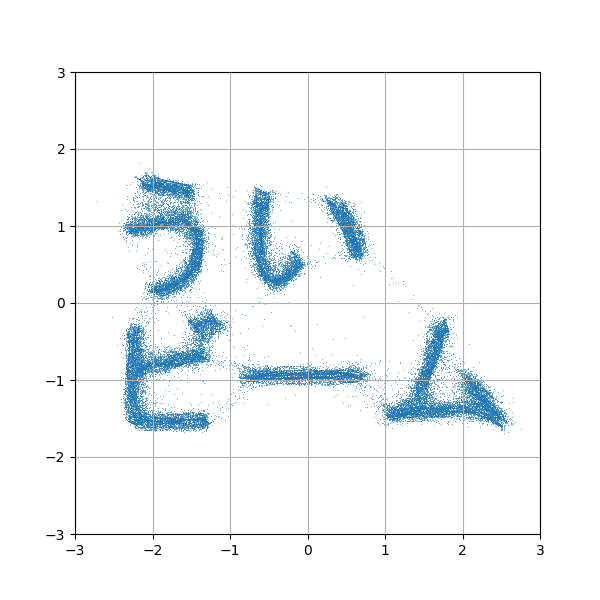} \\
    ITVM $(\mu = 0)$ &
    ITVM $(\mu = 0.9)$ &
    ITVM $(\mu = 0.99)$ &
    ITVM $(\mu = 0.999)$ \\
    & & & \\
    \multicolumn{4}{c}{Samples of student models were generated with 1 function evaluation.}
    \end{tabular}        
\end{center}

\newpage

\subsubsection{2D Dataset Results: SPIRAL}

\begin{center}
    \centering
    \begin{tabular}{|l|r|r|r|r|r|}
        \hline
        \multirow{2}{*}{Loss} 
        & \multicolumn{4}{c|}{KL Divergence ($\downarrow$)} 
        & \multirow{2}{*}{Fused rank}\\
        \cline{2-5}
        & \multicolumn{1}{c|}{NFE = 1} & \multicolumn{1}{c|}{NFE = 2} & \multicolumn{1}{c|}{NFE = 4} & \multicolumn{1}{c|}{NFE = 8} & \\
        \hline
EFMD & 0.140238 (07)& 0.099635 (07)& 0.069224 (07)& 0.042994 (07)& 07\\
LFMD & 0.042539 (04)& 0.028981 (05)& 0.030644 (06)& 0.023447 (06)& 05\\
PID & 0.048297 (05)& 0.023088 (04)& 0.025469 (04)& 0.019121 (04)& 04\\
\hline
ITVM ($\mu = 0.0$) & 0.105905 (06)& 0.038632 (06)& 0.027042 (05)& 0.022703 (05)& 06\\
ITVM ($\mu = 0.9$) & {\color{Blue} 0.036392 (03)}& {\color{Blue} 0.019368 (03)}& {\color{Green} 0.014638 (02)}& {\color{Red} 0.009251 (01)}& {\color{Green} 02}\\
ITVM ($\mu = 0.99$) & {\color{Green} 0.035553 (02)}& {\color{Green} 0.018907 (02)}& {\color{Blue} 0.015143 (03)}& {\color{Blue} 0.010459 (03)}& {\color{Blue} 03}\\
ITVM ($\mu = 0.999$) & {\color{Red} 0.034073 (01)}& {\color{Red} 0.015597 (01)}& {\color{Red} 0.013374 (01)}& {\color{Green} 0.009900 (02)}& {\color{Red} 01}\\
\hline
    \end{tabular}
\end{center}

\begin{center}
    \small
\begin{tabular}{@{\hskip 0.01\linewidth}c@{\hskip 0.01\linewidth}c@{\hskip 0.01\linewidth}c@{\hskip 0.01\linewidth}c@{\hskip 0.01\linewidth}}
    \includegraphics[width=0.23\linewidth]{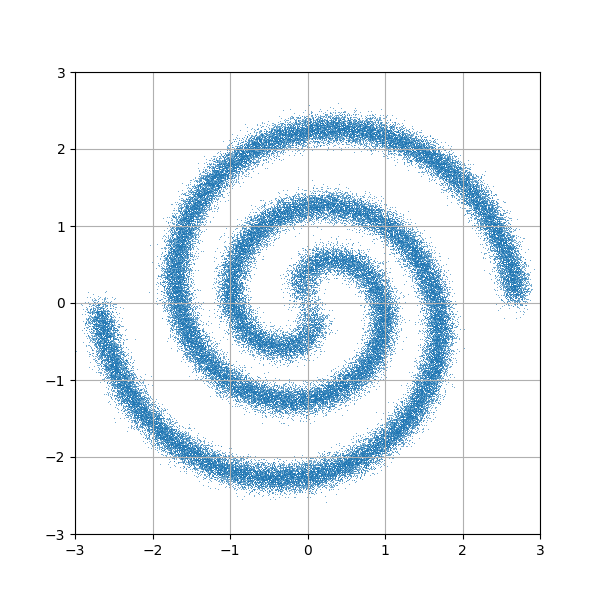} &
    \includegraphics[width=0.23\linewidth]{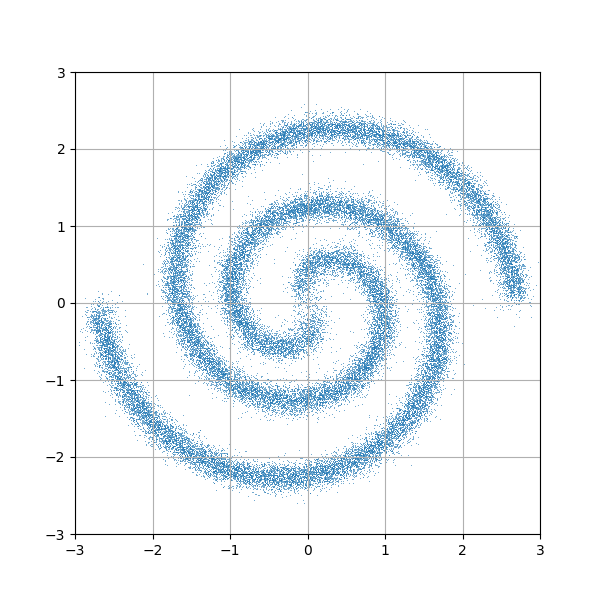} &
    &
    \\
    Dataset &
    Teacher (NFE = 100) &
    &
    \\    
    \includegraphics[width=0.23\linewidth]{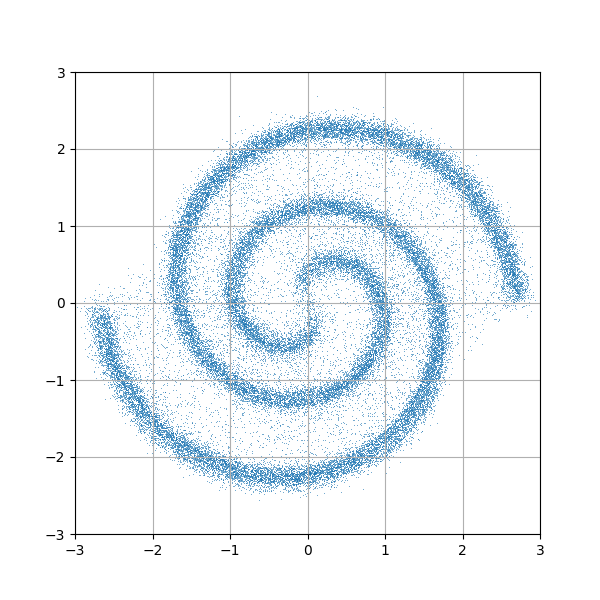} &
    \includegraphics[width=0.23\linewidth]{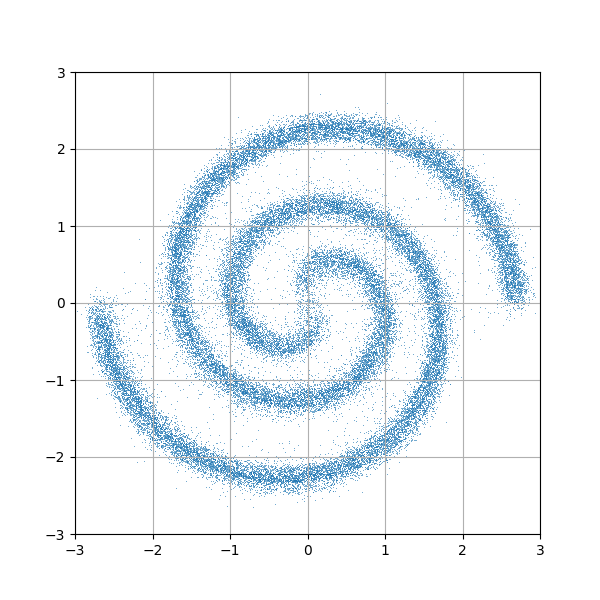} &
    \includegraphics[width=0.23\linewidth]{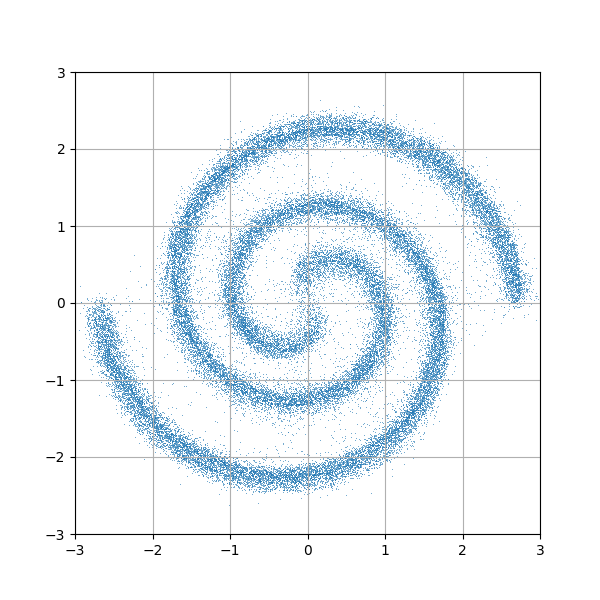} &    
    \\
    EFMD &
    LFMD &
    PID &
    \\
    \includegraphics[width=0.23\linewidth]{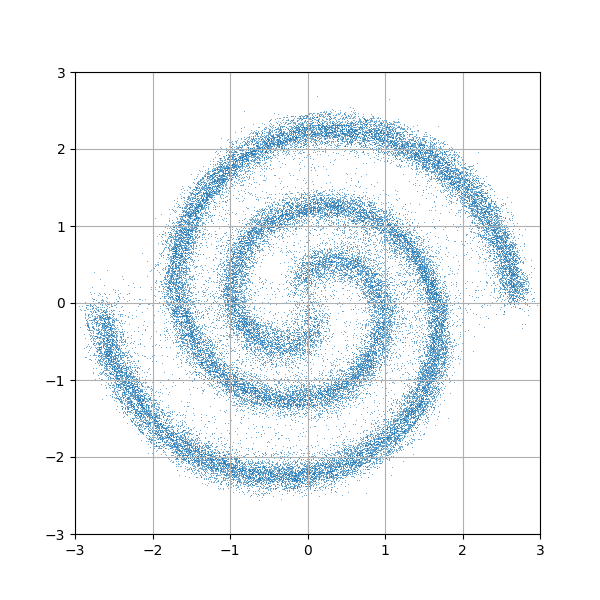} &
    \includegraphics[width=0.23\linewidth]{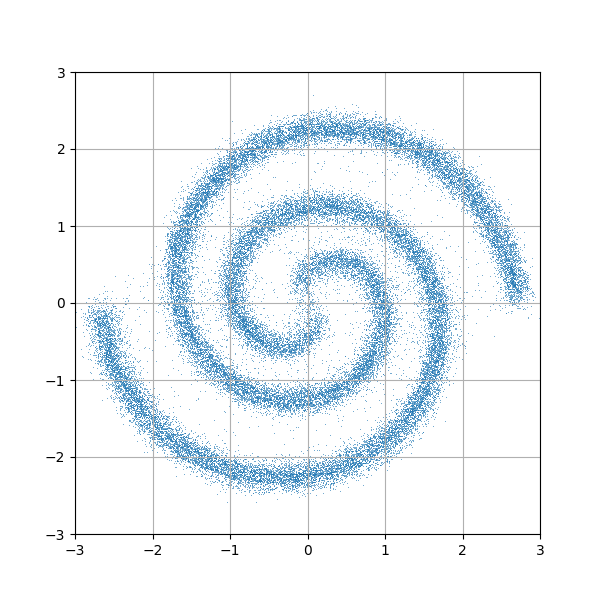} &
    \includegraphics[width=0.23\linewidth]{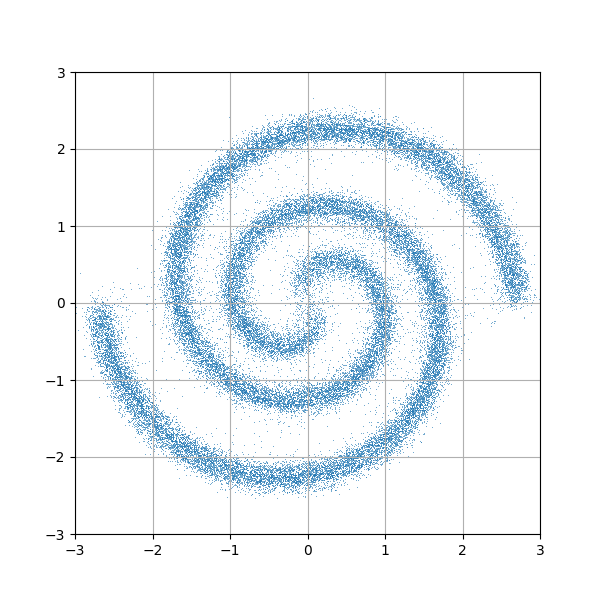} &
    \includegraphics[width=0.23\linewidth]{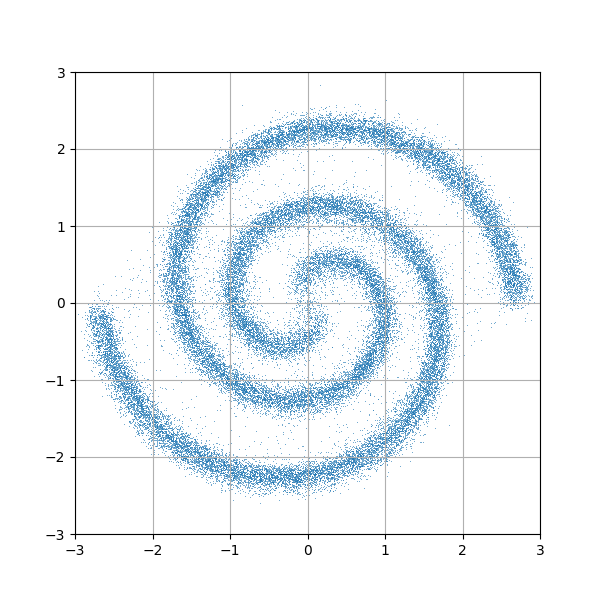} \\
    ITVM $(\mu = 0)$ &
    ITVM $(\mu = 0.9)$ &
    ITVM $(\mu = 0.99)$ &
    ITVM $(\mu = 0.999)$ \\
    & & & \\
    \multicolumn{4}{c}{Samples of student models were generated with 1 function evaluation.}
    \end{tabular}        
\end{center}

\newpage

\subsection{Tabular Datasets}

\subsubsection{Network Architecture}

For tabular datasets, student and teacher models have the exact same architecture for their main bodies, but the main bodies are wrapped differently to conform to the respective models' interfaces. The main body takes in two times $s, t \in \Real$ and a point $x \in \Real^d$ where $d$ is the number of dimensions of each vector in the dataset (8 for GAS, 21 for HEPMASS, and so on). We use an architecture called MLPSkipNet, which we take from the diffusion autoencoder paper by Preechakul \etal~\yrcite{Preechakul:2022}. 

The architecture mainly consists of blocks called ``conditional linear blocks'' (CLB), which is shown in Figure~\ref{fig:conditonal-linear-block}. Such a block takes as input two vectors: a ``feature'' vector with $m$ dimensions, and a ``condition'' vector with $c$ dimensions. The feature vector is transformed by a fully-connected layer to an $n$-dimensional vector. The condition vector is transformed into 2 vectors which are used to scale and shift the transformed feature vector. The output of the block has $n$ dimensions and is obtained by processing the transformed feature vector after scaling and shifting with layer normalization \cite{Ba:2016} and then ELU non-linearity.

\begin{figure}[h]
    \centering
    \includegraphics[width=0.5\linewidth]{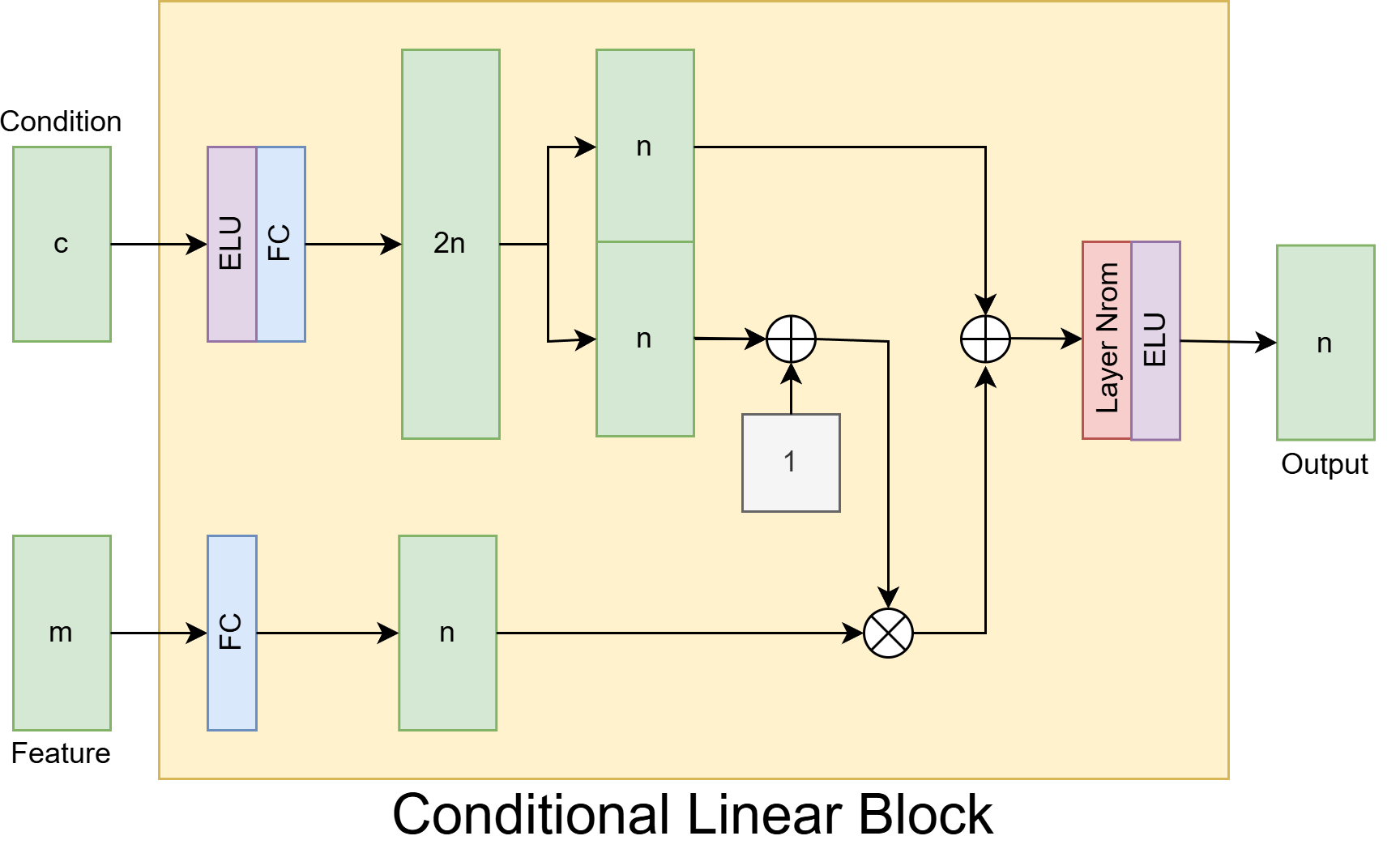}
    \caption{Conditional linear block. Addition ($\oplus$) and multiplication ($\otimes$) nodes denote element-wise addition and element-wise multiplication, respectively.}
    \label{fig:conditonal-linear-block}
\end{figure}

The architecture of an MLPSkipNet is depicted in Figure~\ref{fig:mlp-skipnet}. It has a train of 8 CLBs such that all but the first block receive the concatenation of the previous block's output and the input vector $x$ as its feature vector. These skip connections, which fuse $x$ with intermediate feature vectors, are the reason we call the architecture MLPSkipNet. The output dimension $n$ of each CLB is 2048, and the feature dimension $m$ is $2048 + d$, except for the first CLB whose $m$ is 2048. The condition vector that is fed to all CLB is derived from the two input times. They are converted to 512-dimensional positional encodings, concatenated, and transformed into a $d$-dimensional vector through a 2-layered MLP.

\begin{figure}[h]
    \centering
    \includegraphics[width=0.9\linewidth]{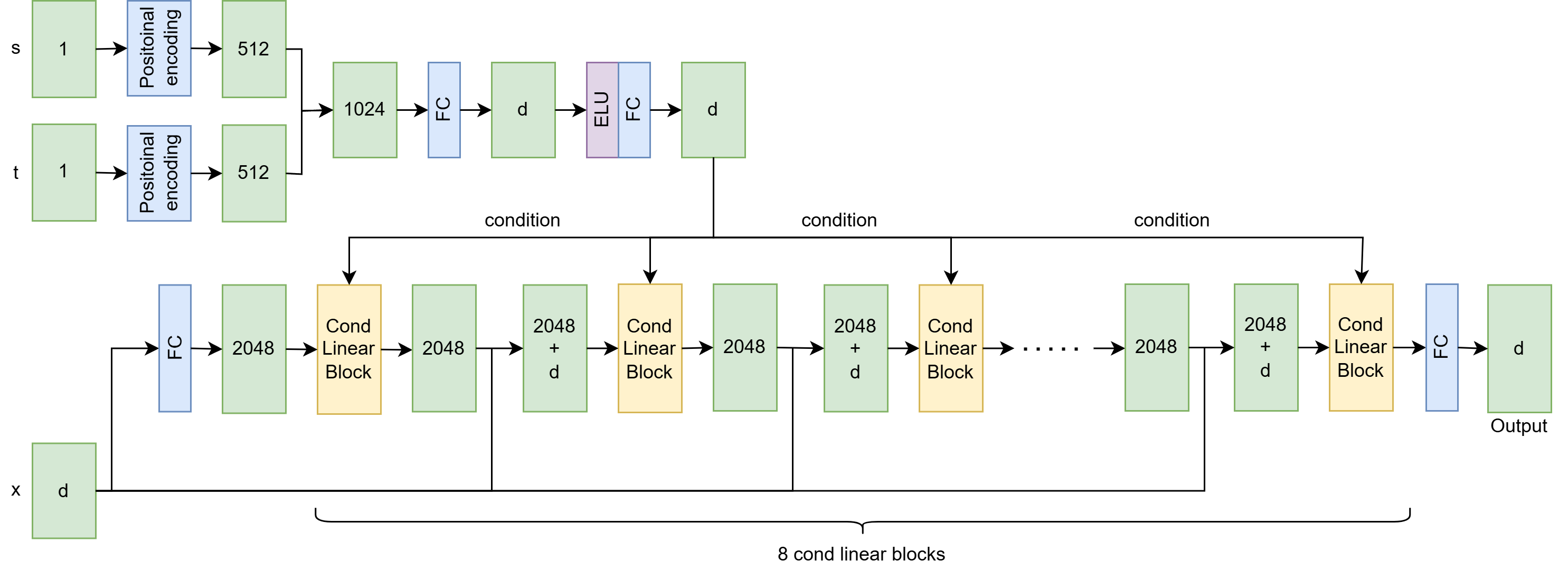}
    \caption{MLPSkipNet architecture.}
    \label{fig:mlp-skipnet}
\end{figure}

A teacher network takes one time variable $t$ as input, and the MLPSkipNet takes 2 time variables. We construct a teacher network by wrapping the MLPSkipNet so that its two input times are set to the input time of the wrapper class. For the student model, the MLPSkipNet serves as its AVM $v^\theta_{s,t}(x)$, which is then wrapped to form $\phi^\theta_{s,t}(x)$ like in Section~\ref{sec:2d-model-architecture}.

Both teacher and student networks are trained using the process described in Section~\ref{sec:2d-model-training-process}.

In the next 4 sessions, we show KL divergence metrics for student models trained with our ITVM loss and other baselines for tabular datasets. Currently, we have not identified an effective way to visualize the samples. 

\subsubsection{Tabular Dataset Results: GAS}

\begin{center}
    \centering
    \begin{tabular}{|l|r|r|r|r|r|}
        \hline
        \multirow{2}{*}{Method} 
        & \multicolumn{4}{c|}{KL Divergence ($\downarrow$)} 
        & \multirow{2}{*}{Fused rank}\\
        \cline{2-5}
        & \multicolumn{1}{c|}{NFE = 1} & \multicolumn{1}{c|}{NFE = 2} & \multicolumn{1}{c|}{NFE = 4} & \multicolumn{1}{c|}{NFE = 8} & \\
        \hline
EFMD & 19.034107 (07)& 27.089544 (07)& 37.149078 (07)& 41.112362 (07)& 07\\
LFMD & 3.481750 (05)& {\color{Blue} 2.193527 (03)}& 1.458288 (05)& 0.835541 (05)& 05\\
PID & {\color{Blue} 2.864193 (03)}& 3.153163 (06)& 1.726937 (06)& 1.017411 (06)& 06\\
\hline
ITVM ($\mu = 0.0$) & 4.084299 (06)& 2.219406 (04)& {\color{Green} 1.287662 (02)}& {\color{Green} 0.681458 (02)}& {\color{Blue} 03}\\
ITVM ($\mu = 0.9$) & {\color{Green} 2.632812 (02)}& 2.383862 (05)& {\color{Red} 1.218426 (01)}& {\color{Red} 0.605913 (01)}& {\color{Green} 02}\\
ITVM ($\mu = 0.99$) & {\color{Red} 2.576506 (01)}& {\color{Red} 1.944362 (01)}& {\color{Blue} 1.300219 (03)}& {\color{Blue} 0.731111 (03)}& {\color{Red} 01}\\
ITVM ($\mu = 0.999$) & 3.106254 (04)& {\color{Green} 1.982998 (02)}& 1.345918 (04)& 0.740087 (04)& 04\\
\hline
    \end{tabular}
\end{center}

\subsubsection{Tabular Dataset Results: HEPMASS}

\begin{center}
    \centering
    \begin{tabular}{|l|r|r|r|r|r|}
        \hline
        \multirow{2}{*}{Method} 
        & \multicolumn{4}{c|}{KL Divergence ($\downarrow$)} 
        & \multirow{2}{*}{Fused rank}\\
        \cline{2-5}
        & \multicolumn{1}{c|}{NFE = 1} & \multicolumn{1}{c|}{NFE = 2} & \multicolumn{1}{c|}{NFE = 4} & \multicolumn{1}{c|}{NFE = 8} & \\
        \hline
EFMD & 3.871845 (07)& 3.329424 (07)& 2.541658 (07)& 2.150262 (07)& 07\\
LFMD & 1.186905 (04)& {\color{Green} 0.939028 (02)}& {\color{Red} 0.358874 (01)}& 0.212761 (04)& {\color{Green} 02}\\
PID & 1.542357 (06)& {\color{Red} 0.834312 (01)}& 0.374184 (04)& 0.217585 (05)& 05\\
\hline
ITVM ($\mu = 0.0$) & 1.348401 (05)& 1.017282 (05)& 0.418560 (06)& 0.235091 (06)& 06\\
ITVM ($\mu = 0.9$) & {\color{Red} 1.091968 (01)}& 1.024881 (06)& 0.374659 (05)& {\color{Blue} 0.205551 (03)}& 04\\
ITVM ($\mu = 0.99$) & {\color{Green} 1.123249 (02)}& {\color{Blue} 0.978086 (03)}& {\color{Green} 0.360455 (02)}& {\color{Green} 0.198744 (02)}& {\color{Red} 01}\\
ITVM ($\mu = 0.999$) & {\color{Blue} 1.184638 (03)}& 0.998316 (04)& {\color{Blue} 0.362117 (03)}& {\color{Red} 0.189735 (01)}& {\color{Blue} 03}\\
\hline
    \end{tabular}
\end{center}

\subsubsection{Tabular Dataset Results: MINIBOONE}

\begin{center}
    \centering
    \begin{tabular}{|l|r|r|r|r|r|}
        \hline
        \multirow{2}{*}{Method} 
        & \multicolumn{4}{c|}{KL Divergence ($\downarrow$)} 
        & \multirow{2}{*}{Fused rank}\\
        \cline{2-5}
        & \multicolumn{1}{c|}{NFE = 1} & \multicolumn{1}{c|}{NFE = 2} & \multicolumn{1}{c|}{NFE = 4} & \multicolumn{1}{c|}{NFE = 8} & \\
        \hline
EFMD & 68.027 (07)& 155.235 (07)& 191.797 (07)& 206.934 (07)& 07\\
LFMD & {\color{Green} 7.927 (02)}& {\color{Green} 5.904 (02)}& {\color{Red} 5.547 (01)}& {\color{Red} 5.566 (01)}& {\color{Red} 01}\\
PID & {\color{Red} 7.918 (01)}& {\color{Red} 5.790 (01)}& {\color{Green} 5.574 (02)}& {\color{Green} 6.013 (02)}& {\color{Red} 01}\\
\hline
ITVM ($\mu = 0.0$) & 10.071 (04)& 8.560 (04)& {\color{Blue} 7.410 (03)}& {\color{Blue} 6.662 (03)}& {\color{Blue} 03}\\
ITVM ($\mu = 0.9$) & {\color{Blue} 10.067 (03)}& {\color{Blue} 8.533 (03)}& 7.449 (04)& 6.707 (04)& {\color{Blue} 03}\\
ITVM ($\mu = 0.99$) & 10.156 (05)& 8.591 (05)& 7.630 (05)& 6.916 (05)& 05\\
ITVM ($\mu = 0.999$) & 10.708 (06)& 8.751 (06)& 7.700 (06)& 7.831 (06)& 06\\
\hline
    \end{tabular}
\end{center}

\subsubsection{Tabular Dataset Results: POWER}

\begin{center}
    \centering
    \begin{tabular}{|l|r|r|r|r|r|}
        \hline
        \multirow{2}{*}{Method} 
        & \multicolumn{4}{c|}{KL Divergence ($\downarrow$)} 
        & \multirow{2}{*}{Fused rank}\\
        \cline{2-5}
        & \multicolumn{1}{c|}{NFE = 1} & \multicolumn{1}{c|}{NFE = 2} & \multicolumn{1}{c|}{NFE = 4} & \multicolumn{1}{c|}{NFE = 8} & \\
        \hline
EFMD & 8.159658 (07)& 9.415152 (07)& 11.877481 (07)& 13.189130 (07)& 07\\
LFMD & 0.209364 (06)& 0.101803 (06)& 0.037712 (04)& 0.020508 (05)& 05\\
PID & 0.173100 (05)& 0.090578 (05)& 0.048503 (06)& 0.027075 (06)& 06\\
\hline
ITVM ($\mu = 0.0$) & 0.162801 (04)& 0.080080 (04)& 0.037747 (05)& 0.019898 (04)& 04\\
ITVM ($\mu = 0.9$) & {\color{Blue} 0.130872 (03)}& {\color{Blue} 0.060825 (03)}& {\color{Blue} 0.026033 (03)}& {\color{Blue} 0.015285 (03)}& {\color{Blue} 03}\\
ITVM ($\mu = 0.99$) & {\color{Red} 0.115969 (01)}& {\color{Red} 0.049607 (01)}& {\color{Green} 0.024108 (02)}& {\color{Green} 0.013504 (02)}& {\color{Red} 01}\\
ITVM ($\mu = 0.999$) & {\color{Green} 0.116765 (02)}& {\color{Green} 0.052908 (02)}& {\color{Red} 0.022147 (01)}& {\color{Red} 0.012813 (01)}& {\color{Red} 01}\\
\hline
    \end{tabular}
\end{center}

\newpage 

\subsection{Image Datasets}

\subsubsection{Network Architectures}

Both teacher and student models have U-Net backbones with attention \cite{Ho:2020} whose implementations are based on those used by Dhariwal and Nichol in their ADM paper \yrcite{Dhariwal:2021}. Their configurations can be found in Table~\ref{tab:image-dataset-architectures}. We would like to point out some differences between architectures of the teacher and the student.

\begin{enumerate}
    \item The teacher uses dropout \cite{Srivastava:2014} layers with probability 0.2, but the student does not. We tried a student architecture with such layers, but we found that samples quality degraded significantly.

    \item Again, the teacher takes only one time variable $t$ as input, but the student takes 2. In the teacher, $t$ is converted to a 1024-dimensional vector via positional encoding, which is then processed by a 2-layered MLP to yield a 1024-dimensional conditioning vector. This conditioning vector is used to condition all ResNet \cite{He:2016} blocks inside the network through Adaptive Instance Normalization (AdaIN) \cite{Huang:2017:AdaIN}. In the student, both the initial time $s$ and the terminal time $t$ are converted to $1024$-dimensional positional encodings. The encodings are then concatenated to form a $2048$-dimensional vector, which is later processed by a 2-layer MLP to yield a 1024-dimensional conditioning vector. How the networks handle time inputs are their main architectural differences.
\end{enumerate}

\begin{table}[h]
\caption{ADM configurations for teacher and student networks for the MNIST and CIFAR-10 datasets.}
    \label{tab:image-dataset-architectures}    
    \centering        
    \begin{tabular}{l|cccc}
        \toprule
        Configuration & MNIST Teacher & MNIST Student & CIFAR-10 Teacher & CIFAR-10 Student \\
        \midrule
        Base \#\,channel & 128 & 128 & 256 & 256  \\
        Channel multipliers & [1,1,1] & [1,1,1] & [1,1,1] & [1,1,1] \\
        \# ResNet blocks per resolution & 3 & 3 & 3 & 3  \\
        Attention resolutions & [14, 7] & [14, 7] & [16,8] & [16,8] \\
        Dropout probability & 0.2 & 0.0 &  0.2 & 0.0 \\
        Size (MB) & 79 & 84 & 237 & 242 \\
        \bottomrule
    \end{tabular}
\end{table}

Again, for the student, its U-Net with attention serves as the AVM $v^\theta_{s,t}(x)$, which must be wrapped to form a TTFM $\phi^{\theta}_{s,t}(x)$.

\subsubsection{Training Processes}

Both the teacher and student were trained using processes similar to those described in Section~\ref{sec:2d-model-training-process}. The differences were training lengths, batch sizes, and the EMA decay rates used to create test-time models.
\begin{itemize}
    \item For MNIST, batch sizes were 128 for all models. Training lasted $64M$ training examples, which amounted to 500,000 iterations and approximately 1,067 epochs.
    \item For CIFAR-10, batch sizes were 80 for most models, except for student models trained with EFMD and LFMD. For these models, we had to reduce the batch size to 56 because 80 required more RAM than our GPUs had available. Like MNIST, training also lasted 64M training examples (1,280 epochs).
    \item The EMA decay rate used to create test-time models was 0.9999 instead of 0.999.
\end{itemize}

\newpage

\subsubsection{Image Dataset Results: MNIST}

\newpage

\begin{center}
    \begin{tabular}{|l|r|r|r|r|r|}
        \hline
        \multirow{2}{*}{Method} 
        & \multicolumn{4}{c|}{FID ($\downarrow$)} 
        & \multirow{2}{*}{Fused rank}\\
        \cline{2-5}
        & \multicolumn{1}{c|}{NFE = 1} & \multicolumn{1}{c|}{NFE = 2} & \multicolumn{1}{c|}{NFE = 4} & \multicolumn{1}{c|}{NFE = 8} & \\
        \hline
EFMD & 17.866039 (07)& 16.011790 (07)& 16.131295 (07)& 18.354233 (07)& 07\\
LFMD & 12.279407 (06)& 3.500427 (06)& 1.170980 (06)& 0.869458 (06)& 06\\
PID & {\color{Blue} 2.338407 (03)}& 1.758009 (05)& {\color{Blue} 0.842187 (03)}& {\color{Red} 0.586441 (01)}& {\color{Green} 02}\\
\hline
ITVM ($\mu = 0.0$) & 2.664168 (05)& 1.471545 (04)& {\color{Green} 0.820349 (02)}& {\color{Blue} 0.623237 (03)}& 04\\
ITVM ($\mu = 0.9$) & 2.532360 (04)& {\color{Blue} 1.461902 (03)}& 0.959420 (05)& 0.637742 (04)& 05\\
ITVM ($\mu = 0.99$) & {\color{Red} 2.174508 (01)}& {\color{Red} 1.167029 (01)}& {\color{Red} 0.812260 (01)}& {\color{Green} 0.603876 (02)}& {\color{Red} 01}\\
ITVM ($\mu = 0.999$) & {\color{Green} 2.301054 (02)}& {\color{Green} 1.184924 (02)}& 0.869098 (04)& 0.670455 (05)& {\color{Blue} 03}\\
\hline
    \end{tabular}    
\end{center}

\begin{center}
    \small
\begin{tabular}{@{\hskip 0.01\linewidth}c@{\hskip 0.01\linewidth}c@{\hskip 0.01\linewidth}c@{\hskip 0.01\linewidth}c@{\hskip 0.01\linewidth}c@{\hskip 0.01\linewidth}}
    \includegraphics[width=0.23\linewidth]{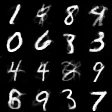} &
    \includegraphics[width=0.23\linewidth]{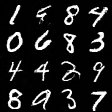} &
    \includegraphics[width=0.23\linewidth]{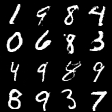} &    
    \\
    EFMD &
    LFMD &
    PID \\
    \includegraphics[width=0.23\linewidth]{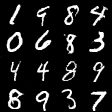} &
    \includegraphics[width=0.23\linewidth]{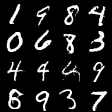} &
    \includegraphics[width=0.23\linewidth]{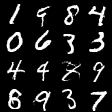} &
    \includegraphics[width=0.23\linewidth]{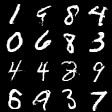} \\
    ITVM $(\mu = 0)$ &
    ITVM $(\mu = 0.9)$ &
    ITVM $(\mu = 0.99)$ &
    ITVM $(\mu = 0.999)$ \\
    & & & \\
    \multicolumn{4}{c}{Samples were generated with 1 function evaluation.} \\
\end{tabular}        
\end{center}

\newpage

\subsubsection{Image Dataset Results: CIFAR-10}

\begin{center}
    \begin{tabular}{|l|r|r|r|r|r|}
        \hline
        \multirow{2}{*}{Method} 
        & \multicolumn{4}{c|}{FID ($\downarrow$)} 
        & \multirow{2}{*}{Fused rank}\\
        \cline{2-5}
        & \multicolumn{1}{c|}{NFE = 1} & \multicolumn{1}{c|}{NFE = 2} & \multicolumn{1}{c|}{NFE = 4} & \multicolumn{1}{c|}{NFE = 8} & \\
        \hline
EFMD & 57.029638 (07)& 42.550134 (07)& 29.906570 (07)& 26.798251 (07)& 07\\
LFMD & 12.820038 (06)& 6.939151 (05)& {\color{Green} 4.766512 (02)}& {\color{Red} 4.194140 (01)}& {\color{Blue} 03}\\
PID & 12.730561 (05)& 7.117951 (06)& 5.025122 (04)& 4.555006 (04)& 05\\
\hline
ITVM ($\mu = 0.0$) & 9.979253 (04)& 6.903587 (04)& 5.147546 (05)& 4.690248 (06)& 05\\
ITVM ($\mu = 0.9$) & {\color{Red} 9.376229 (01)}& {\color{Blue} 6.587607 (03)}& 5.161658 (06)& 4.600747 (05)& 04\\
ITVM ($\mu = 0.99$) & {\color{Blue} 9.734444 (03)}& {\color{Green} 6.269003 (02)}& {\color{Blue} 4.829522 (03)}& {\color{Blue} 4.482476 (03)}& {\color{Green} 02}\\
ITVM ($\mu = 0.999$) & {\color{Green} 9.531846 (02)}& {\color{Red} 6.217342 (01)}& {\color{Red} 4.751021 (01)}& {\color{Green} 4.407881 (02)}& {\color{Red} 01}\\
\hline
    \end{tabular}    
\end{center}

\begin{center}
    \small
\begin{tabular}{@{\hskip 0.01\linewidth}c@{\hskip 0.01\linewidth}c@{\hskip 0.01\linewidth}c@{\hskip 0.01\linewidth}c@{\hskip 0.01\linewidth}c@{\hskip 0.01\linewidth}}
    \includegraphics[width=0.24\linewidth]{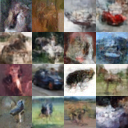} &
    \includegraphics[width=0.24\linewidth]{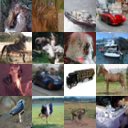} &
    \includegraphics[width=0.24\linewidth]{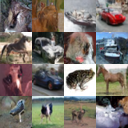} &    
    \\
    EFMD &
    LFMD &
    PID \\
    \includegraphics[width=0.23\linewidth]{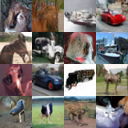} &
    \includegraphics[width=0.23\linewidth]{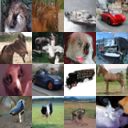} &
    \includegraphics[width=0.23\linewidth]{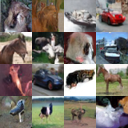} &
    \includegraphics[width=0.23\linewidth]{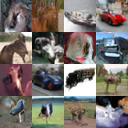} \\
    ITVM $(\mu = 0)$ &
    ITVM $(\mu = 0.9)$ &
    ITVM $(\mu = 0.99)$ &
    ITVM $(\mu = 0.999)$ \\
    & & & \\
    \multicolumn{4}{c}{Samples were generated with 1 function evaluation.} \\
\end{tabular}        
\end{center}

\newpage

\section{A Study on Choices of Intermediate Time in the TVM loss} \label{sec:full-u-ablation}

\begin{figure}
    \centering
    \begin{tabular}{@{\hskip 0.01\linewidth}c@{\hskip 0.01\linewidth}c@{\hskip 0.01\linewidth}c@{\hskip 0.01\linewidth}c@{\hskip 0.01\linewidth}}
        \includegraphics[width=0.24\linewidth]{images/u_ablation/twodim_01/kl_0000.png} &
        \includegraphics[width=0.24\linewidth]{images/u_ablation/twodim_01/kl_0001.png} &
        \includegraphics[width=0.24\linewidth]{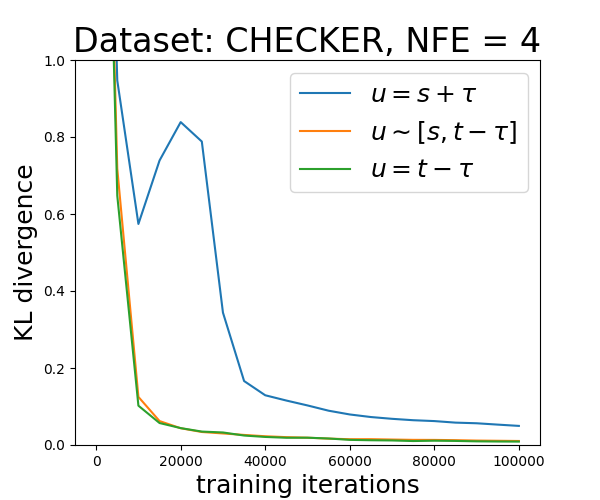} &
        \includegraphics[width=0.24\linewidth]{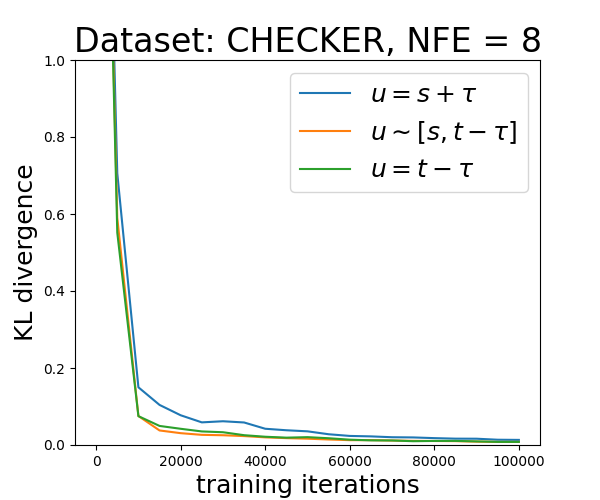} \\
        \includegraphics[width=0.24\linewidth]{images/u_ablation/twodim_02/kl_0000.png} &
        \includegraphics[width=0.24\linewidth]{images/u_ablation/twodim_02/kl_0001.png} &
        \includegraphics[width=0.24\linewidth]{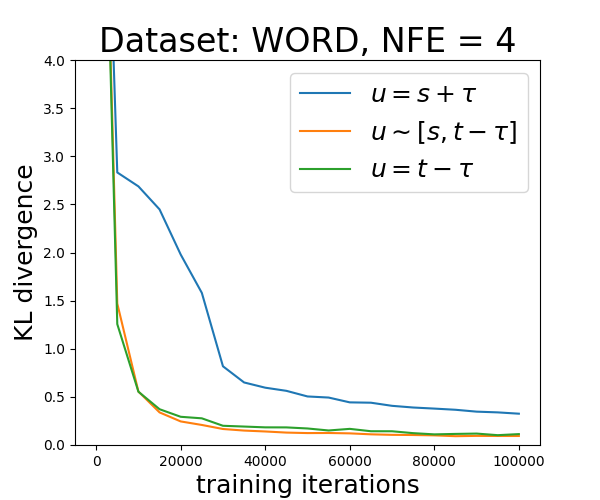} &
        \includegraphics[width=0.24\linewidth]{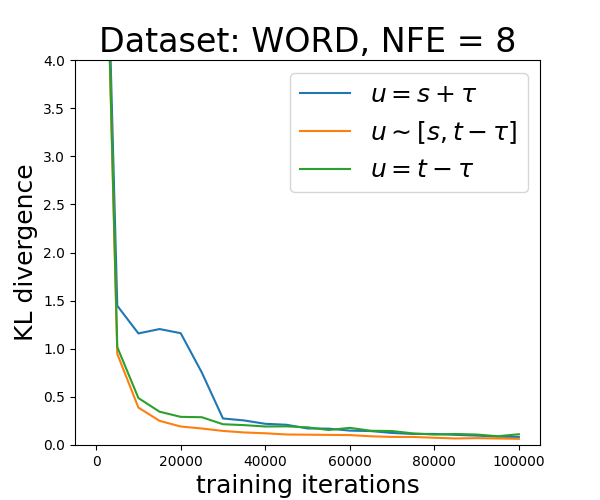} 
    \end{tabular}    
    \vspace{-0.7em}
    \caption{The effect of how the intermediate time $u$ is chosen on the evolution of the KL divergence between distributions of teacher models and student models during training.}
    \label{fig:full-u-ablation}
\end{figure}

The consistency property states: $\phi_{s,t}^\theta(x) = \phi_{u,t}^\theta(\phi_{s,u}^\theta(x))$ for all $0 \leq s \leq u \leq t$. It involves three time variables: the initial time $s$, the intermediate time $u$, and the terminal time $t$. The TVM loss tries to enforce the property by minimizing the L2 norm between the LHS and the RHS of the equation, and it chooses $u = t-\tau$ where $\tau > 0$ is a small constant. We are interested to see whether this choice is optimal or not.

Before we pursue the question further, let us rewrite the TVM loss in such a way that one may change how the intermediate time is chosen.
\begin{align}
    \mcal{L}_{\operatorname{TVM*}} = E_{\substack{s \sim [0,1],\\t \sim [s,1],\\ x_s \sim p_s}} \bigg[ \bigg\| \frac{\phi^{\theta}_{s,t}(x_s) - \phi^{\theta}_{s,u}(x_s) }{t-u} - v^{\langle \theta \rangle}_{u,t}(\phi_{s,u}^{[\theta]}(x_s)) \bigg\|^2 \bigg]. \label{eqn:modified-tvm}
\end{align}

We compare three strategies of choosing the intermediate time $u$.
\begin{enumerate}
    \item Choosing $u = s+\tau$, which is employed by CM.
    \item Sampling $u \sim \mcal{U}[s,t]$, which is employed by SCM.
    \item Choosing $u = t - \tau$, which is employed by our proposed TVM.
\end{enumerate}
Here, $\tau$ is a small constant, which is fixed to $0.005$. When implementing the second strategy, we set the upper limit of $u$ to $t-\tau$ instead of $t$ to prevent the denominator $t-u$ on the RHS of \eqref{eqn:modified-tvm} to be too small. The EMA decay rate $\mu$ was fixed to $0.99$. For each of the CHECKER and WORD datasets, we trained 3 student models with 3 versions of ITVM where the TVM term chooses $u$ according to one of the three strategies. We show plots of the KL divergence metrics at NFE = $1$, $2$, $4$, and $8$ as functions of training iterations in Figure~\ref{fig:full-u-ablation}.

From the plots, we can observe the following trends.

\begin{itemize}
    \item Picking $u = s+\tau$ performs significantly worse than the other two strategies when NFE = 1, 2, and 4. The gap vanishes at NFE = 8, but models trained with this strategy seem to converge slower.

    \item $u = t-\tau$ results in better performance than $u \sim [s,t]$ when NFE = 1, but the performance gap disappears at higher NFEs. Moreover, $u \sim [s,t]$ seems to lead to faster convergence when NFE = 4 and 8.
\end{itemize}

We conclude that the $u = t - \tau$ strategy is superior to other strategies because (1) it leads to noticeably better performance when NFE = 1, and (2) it performs on par with other strategies at higher NFEs.

\section{A Study on the Step Size $\tau$} \label{sec:ful-tau-ablation}

\begin{figure}
    \centering
    \begin{tabular}{@{\hskip 0.01\linewidth}c@{\hskip 0.01\linewidth}c@{\hskip 0.01\linewidth}c@{\hskip 0.01\linewidth}c@{\hskip 0.01\linewidth}}
        \includegraphics[width=0.24\linewidth]{images/tau_ablation/twodim_01/kl_0000.png} &
        \includegraphics[width=0.24\linewidth]{images/tau_ablation/twodim_01/kl_0001.png} &
        \includegraphics[width=0.24\linewidth]{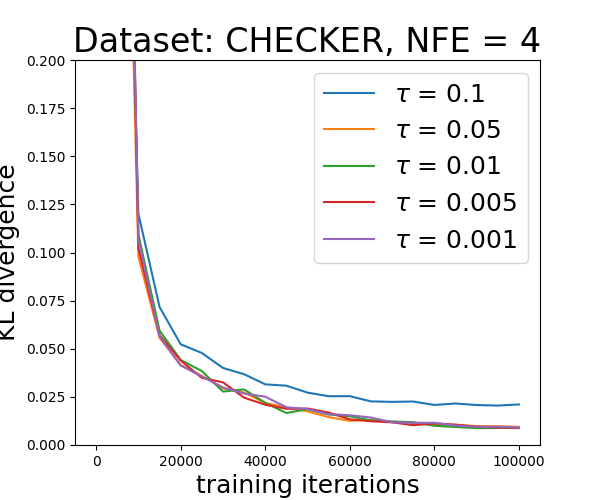} &
        \includegraphics[width=0.24\linewidth]{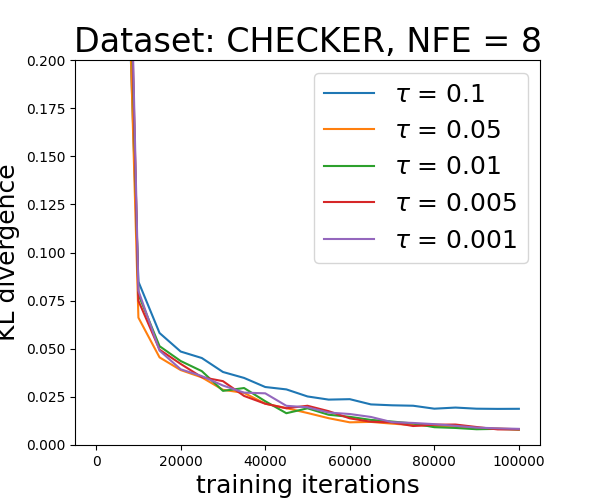} \\
        \includegraphics[width=0.24\linewidth]{images/tau_ablation/twodim_02/kl_0000.png} &
        \includegraphics[width=0.24\linewidth]{images/tau_ablation/twodim_02/kl_0001.png} &
        \includegraphics[width=0.24\linewidth]{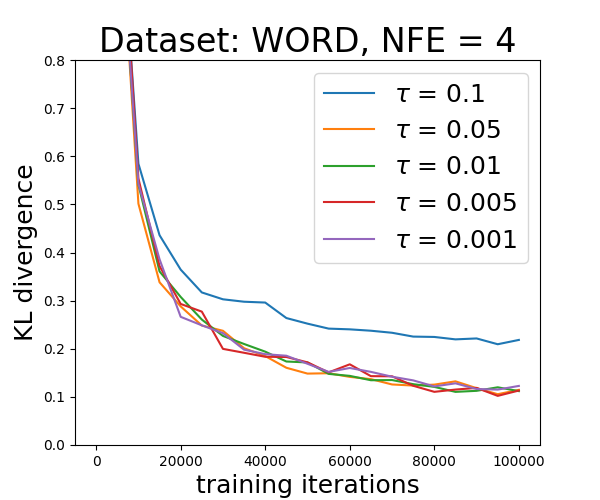} &
        \includegraphics[width=0.24\linewidth]{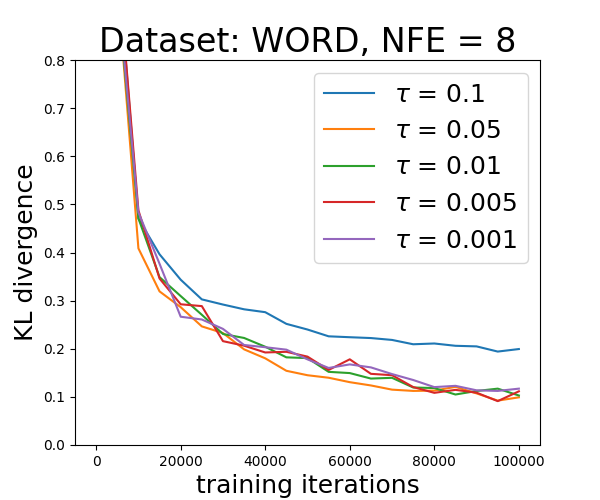} 
    \end{tabular}    
    \caption{The effect of the step size $\tau$ on the evolution of the KL divergence between distributions of teacher models and student models during training.}
    \label{fig:full-tau-ablation}
\end{figure}

We study the effect of the hyperparameter $\tau$ by training student models on the CHECKER and WORD datasets with ITVM, setting $\tau$ to $0.1$, $0.05$, $0.01$, $0.005$, and $0.001$. We show the evaluation of the KL divergence metric as training progresses in Figure~\ref{fig:full-tau-ablation}. We find that almost all values of $\tau$ yield comparable convergence rates and final performance. The exception is $\tau = 0.1$, which performed worse than other values at most dataset-NFE combinations. We conclude that one should pick a $\tau$ value that is small enough, and we find $0.005$ to be a good value to use.


\end{document}